\documentclass{article}

\usepackage{microtype}
\usepackage{graphicx}
\usepackage{subfigure}
\usepackage{booktabs} % for professional tables

\usepackage{hyperref}

\usepackage[accepted]{lib/icml2021}

\usepackage{amsmath}
\usepackage{amssymb}
\usepackage{mathtools}
\usepackage{resizegather}
\usepackage{thm-restate}
\usepackage{xspace}
\usepackage{mathrsfs}
\usepackage{bbold}

\usepackage{algorithm}
\usepackage{algpseudocode}

\usepackage{proof-at-the-end}

\newcommand{\sampled}{\overset{{}_\$}{\leftarrow}}

\newcommand{\size}[1]{\left | #1 \right |}
\newcommand{\norm}[1]{\left \| #1 \right \|}

\newcommand{\vect}[1]{\mathbf{#1}}
\newcommand{\mat}[1]{\mathbf{#1}}

\newcommand{\sep}{\;||\;}

\renewcommand{\L}{\mathcal{L}}
\newcommand{\D}{\mathcal{D}}
\newcommand{\A}{\mathcal{A}}
\newcommand{\C}{\mathcal{C}}
\newcommand{\X}{\mathcal{X}}
\newcommand{\B}{\mathcal{B}_k}

\newcommand{\gauss}[1]{\mathcal{N}(0, #1 \vect I_d)}
\newcommand{\sigmoid}{\boldsymbol \sigma}

\DeclareMathOperator*{\E}{\mathbb{E}}
\newcommand{\R}{\mathbb{R}}
\newcommand{\proj}{\Pi_{\C}}
\newcommand{\NoisySGD}{DP-SGLD}
\newcommand{\AlgoNoisySGD}{\mathcal{A}_{\mathrm{DP-SGLD}}}

\newcommand{\eigenmax}{\lambda_{\mathsf{max}}}

\usepackage{amsthm}

\theoremstyle{definition}
\newtheorem{definition}{Definition}[section]
\newtheorem{proposition}{Proposition}[section]
\newtheorem{lemma}{Lemma}[section]

\newtheorem{theorem}{Theorem}[section]
\newtheorem{corollary}{Corollary}[section]

\NewDocumentEnvironment{propositionE}{O{}O{}+b}{%
    \begin{theoremEnd}[normal, end, restate, one big link translated=The complete proof is provided in Appendix on page, #2]{proposition}[#1]%
        #3%
    \end{theoremEnd}%
}{}

\NewDocumentEnvironment{lemmaE}{O{}O{}+b}{%
    \begin{theoremEnd}[normal, end, restate, one big link translated=The proof is provided in Appendix on page, #2]{lemma}[#1]%
        #3%
    \end{theoremEnd}%
}{}

\NewDocumentEnvironment{claimE}{O{}O{}+b}{%
    \begin{theoremEnd}[normal, end, restate, one big link translated=The complete proof is provided in Appendix on page, #2]{claim}[#1]%
        #3%
    \end{theoremEnd}%
}{}

\NewDocumentEnvironment{theoremE}{O{}O{}+b}{%
    \begin{theoremEnd}[normal, end, restate, one big link translated=The complete proof is provided in Appendix on page, #2]{theorem}[#1]%
        #3%
    \end{theoremEnd}%
}{}

\NewDocumentEnvironment{proofE}{O{}+b}{%
    \begin{proofEnd}[#1]%
        #2%
    \end{proofEnd}%
}{}

\icmltitlerunning{Differential Privacy Guarantees for Stochastic Gradient Langevin Dynamics}

\begin{document}

\twocolumn[
\icmltitle{Differential Privacy Guarantees for \\ 
    Stochastic Gradient Langevin Dynamics}

% It is OKAY to include author information, even for blind
% submissions: the style file will automatically remove it for you
% unless you've provided the [accepted] option to the icml2021
% package.

% List of affiliations: The first argument should be a (short)
% identifier you will use later to specify author affiliations
% Academic affiliations should list Department, University, City, Region, Country
% Industry affiliations should list Company, City, Region, Country

% You can specify symbols, otherwise they are numbered in order.
% Ideally, you should not use this facility. Affiliations will be numbered
% in order of appearance and this is the preferred way.
\icmlsetsymbol{equal}{*}

\begin{icmlauthorlist}
\icmlauthor{Théo Ryffel}{inria_ens}
\icmlauthor{Francis Bach}{inria_ens}
\icmlauthor{David Pointcheval}{ens_inria}
\end{icmlauthorlist}

\icmlaffiliation{inria_ens}{INRIA, Département d’informatique de l’ENS, ENS, CNRS, PSL University, Paris, France}
\icmlaffiliation{ens_inria}{Département d’informatique de l’ENS, ENS, CNRS, PSL University, INRIA, Paris, France}

\icmlcorrespondingauthor{Théo Ryffel}{theo.ryffel@ens.fr}

% You may provide any keywords that you
% find helpful for describing your paper; these are used to populate
% the "keywords" metadata in the PDF but will not be shown in the document
\icmlkeywords{Machine Learning, Differential Privacy, Langevin Diffusion, Rényi, ICML}

\vskip 0.3in
]

% this must go after the closing bracket ] following \twocolumn[ ...

% This command actually creates the footnote in the first column
% listing the affiliations and the copyright notice.
% The command takes one argument, which is text to display at the start of the footnote.
% The \icmlEqualContribution command is standard text for equal contribution.
% Remove it (just {}) if you do not need this facility.

\printAffiliationsAndNotice{}  % leave blank if no need to mention equal contribution
%\printAffiliationsAndNotice{\icmlEqualContribution} % otherwise use the standard text.

\begin{abstract}
We analyse the privacy leakage of noisy stochastic gradient descent by modeling Rényi divergence dynamics with Langevin diffusions. Inspired by recent work on non-stochastic algorithms, we derive similar desirable properties in the stochastic setting. In particular, we prove that the privacy loss converges exponentially fast for smooth and strongly convex objectives under constant step size, which is a significant improvement over previous DP-SGD analyses. We also extend our analysis to arbitrary sequences of varying step sizes and derive new utility bounds. Last, we propose an implementation and our experiments show the practical utility of our approach compared to classical DP-SGD libraries.
\end{abstract}
\vspace{-0.5cm}
%\textcolor{red}{We also show that we reach optimal utility with modest computational complexity.} 

\section{Introduction}\label{sec:introduction}

Differential privacy \cite{dwork2014algorithmic} for machine learning is a promising approach to reduce exposure of training datasets when releasing machine learning models. The privacy leakage from these models can be quantified using Rényi differential privacy \cite{mironov2017renyi} which models it through the divergence of the distributions of two models trained on datasets that only differ in one item.
%by providing an upper bound of the leakage modeled through the divergence of the distributions of two models trained on neighboring datasets, i.e., datasets that only differ in one item. 
The intuition behind is that a model whose behavior is sensitive to the presence or absence of a single individual is likely to memorize information about specific individuals, which can then be uncovered using several types of attacks like membership inference attacks \cite{shokri2017membership}.

The most standard approaches to training neural networks with differential privacy are derived from \citet{abadi2016deep}’s method of differentially private stochastic gradient descent (DP-SGD). DP-SGD is an attractive method as it closely mimics classic SGD training of neural networks, and applies to almost all architectures. It therefore enjoys easy adoption from data scientists
%: at each batch update, Gaussian noise is carefully added to the gradient update to hide the contribution of the specific data items. This means that virtually any neural network can be trained using DP-SGD\footnote{There are some slight adjustments to be made for some architectures, as for example BatchNorm is usually replaced with GroupNorm} 
and has been integrated in popular libraries like Opacus \cite{opacus}. Differential privacy of the whole training mechanism is computed using the strong composition theorem \cite{dwork2010boosting}, which states that the privacy leakage modeled through standard $(\epsilon, \delta)$-differential privacy grows approximately in $\sqrt{K}$ for high privacy regimes, where $K$ is the number of iterations. This is a strong limitation of DP-SGD in real world applications since training for a large number of iterations would lead to a prohibitive privacy bound.

Recently, \citet{chourasia2021differential} have proposed a novel analysis of the differential privacy dynamics of Langevin diffusion, resulting in a new differentially private noisy gradient descent algorithm (DP-GLD). This method notably guarantees that under strongly convex and smooth objectives, the privacy leakage can be bounded by a constant, which allows for an unlimited number of model updates. The key difference with DP-SGD is that the model is assumed to be hidden during training and released only once the training is over. This is a setting that we typically encounter while training a model using multi-party computation \cite{knott2021crypten, wagh2021falcon, ryffel2022ariann} where only the final version of the model is visible. However, the algorithm and the privacy analysis provided by \citet{chourasia2021differential} only addresses full gradient descent (DP-GLD) which is impractical for large datasets 
%as the gradient needs to be computed on all the data items at each iteration. This makes
and makes its adoption by data scientists or standard differential privacy libraries less probable.

\textbf{Contributions.} We provide a \emph{stochastic} version of the noisy gradient descent algorithm (DP-SGLD) and build a privacy analysis based on Langevin diffusion. We prove that DP-SGLD achieves similar privacy and utility guarantees than DP-GLD, including exponential convergence of the privacy bound, and we extend the analysis to the case where the step size is not constant.
% We prove that DP-SGLD achieves the same privacy and utility guarantees as DP-GLD \textcolor{red}{with much less computation}. 
% In addition, we revisit the utility proofs of \cite{chourasia2021differential} to provide better utility guarantees.
%In addition, we extend our analysis to the case where the step size is not constant during training. 
%Last, we provide experimental results and comparisons of DP-SGLD with standard DP-SGD and illustrate that DP-SGLD provides significant improvements in terms of utility.
More specifically:
\begin{itemize}
    \item We introduce DP-SGLD, a stochastic version of the DP-GLD algorithm studied by \citet{chourasia2021differential}, and we show that it achieves the same privacy guarantees including exponentially fast convergence.
    \item We show that DP-SGLD achieves similar utility than DP-GLD up to a term due to using stochastic estimates of the gradient. We also relax assumptions on the step size $\eta$ in utility theorems of \citet{chourasia2021differential} to only verify  $\eta \le \frac{1}{2 \beta}$ instead of $\eta \le \frac{\lambda}{2 \beta^2}$, where $\beta$ is the smoothness constant and $\lambda$ the strong-convexity parameter, thus obtaining the classical scaling from convex optimization.
    \item We extend our analysis of DP-SGLD to non-constant step sizes and derive utility bounds when the step size is parametrized as $\eta_k = \frac{1}{2 \beta + \lambda k / 2}$ and removes the term due to stochastic training.
    \item Last, we provide an implementation of DP-SGLD\footnote{The code is provided in the supplementary material.} and an experimental evaluation where we train differentially private logistic regressions on several datasets. We show that DP-SGLD achieves higher experimental accuracy than DP-SGD on these tasks, and that it almost closes the gap with non-private training. We also show that standard DP-SGD does not benefit from training deeper networks on these tasks, which aligns with the conclusions drawn by \citet{dp_alexnet_moment}.
\end{itemize}

Note that our algorithm ends up being similar to the one of \citet{welling2011bayesian}, except that we do not try to construct samples from the posterior distribution but instead to derive privacy guarantees.

%The paper is organised as follows: Section \ref{sec:preliminaries} provides basic definitions about differential privacy, Section \ref{sec:privacy_SGLD} introduces the DP-SGLD algorithm and performs the privacy analysis of the mechanism, Section \ref{sec:utility_SGLD} provides a utility analysis of DP-SGLD and improved guarantees on DP-GLD utility. Last, Section \ref{sec:experiments} provides experiments where we train logistic regression on vision datasets and demonstrates the practical interest for DP-SGLD.
% \note{FB: last paragraph not really needed}

\section{Preliminaries}\label{sec:preliminaries}

Let us first recall the standard definition of $(\epsilon, \delta)$-differential privacy, as introduced by \citet{dwork2014algorithmic} :

\begin{definition}[$(\epsilon, \delta)$-differential privacy]
    A randomized algorithm $\A : \D \mapsto \R^d$ satisfies $(\epsilon, \delta)$-differential privacy if for any neighboring datasets $\D$ and $\D'$, i.e., datasets that only differ in one item, and any subset $S \in \R^d$, the distribution of $\A$ satisfies:
    $$ \mathrm{P}[\A(\D)] \le e^\epsilon \mathrm{P}[\A(\D')] + \delta.$$
\end{definition}

An alternative notion, coined as \emph{Rényi differential privacy} has been proposed by \citet{mironov2017renyi}, which is more suited to studying composition mechanisms, but can be converted back to standard $(\epsilon, \delta)$-differential privacy.

\begin{definition}[Rényi differential privacy]
    A randomized algorithm $\A : \D \mapsto \R^d$ satisfies $(\alpha, \varepsilon)$-Rényi differential privacy if for any neighboring datasets $\D$ and $\D'$, the $\alpha$ Rényi divergence satisfies  $R_\alpha(\A(\D)\sep \A(\D')) \le \varepsilon$, where:
    $$ R_\alpha(\A(\D)\sep \A(\D'))  = \dfrac{1}{\alpha \!-\! 1} \log \E_{\theta \sim \A(\D')} \! \left [ \! \left ( \! \dfrac{\mu_{\A(\D)}(\theta)}{\mu_{\A(\D')}(\theta)} \! \right )^\alpha \right ],$$
    and where $\mu_\A$ denotes the density $\A$.
\end{definition}

Conversion from Rényi differential privacy to $(\epsilon, \delta)$-differential privacy is given by the following proposition:

\begin{proposition}[From Rényi to $(\epsilon, \delta)$-differential privacy  \cite{mironov2017renyi}]\label{prop:renyi_to_dp}
    If $\A$ satisfies $(\alpha, \varepsilon)$-Rényi differential privacy, it also satisfies $(\epsilon, \delta)$-differential privacy for any $0 < \delta < 1$ with
    $$ \epsilon = \varepsilon + \dfrac{\log(1 / \delta)}{\alpha-1} $$ 
\end{proposition}

\section{Privacy analysis of noisy stochastic gradient descent}\label{sec:privacy_SGLD}

We use the same notations as \citet{chourasia2021differential}. Let $\D = (\vect x_1, \dots, \vect x_n)$ be a dataset of size $n$, with $\vect x_i \in \R^p$. %with records taken from a universe $\X$. 
Let $\ell(\theta, \vect x)$ be the loss function of a learning algorithm parametrized by $\theta \in \C$ on an input $x$, where $\C$ is a closed convex set of $\R^d$. $\proj$ denotes the orthogonal projection onto~$\C$.
We denote by $\L_\D(\theta)$ the global empirical loss of the model, and by $\L_{\B}(\theta)$ the estimated empirical loss computed on the batch $\B$ of size $\size{\B} = m$.
\begin{equation*}%\label{eq:loss}
   \!\!\! \L_\D(\theta) = \frac{1}{n} \sum_{\vect x \in \D} \ell(\theta, \vect x) \;\; \:\:\:\:  \ \    \L_{\B}(\theta) = \frac{1}{m} \sum_{\vect x \in \B} \ell(\theta, \vect x).
\end{equation*}
\begin{algorithm}
\caption{$\AlgoNoisySGD$: Noisy Stochastic Gradient Descent}\label{algo:noisy_sgd}
\begin{algorithmic}
\Require Dataset $\D = (\vect x_1, \dots, \vect x_n)$, loss function $\ell$, step size $\{\eta_k\}_{k \ge 0}$, noise variance $\sigma^2$ and initial parameter $\theta_0 \in \C$

\For{$k = 0, \dots, K-1$}
    \State Sample batch $\B$ from $\D$ with replacement 
    \State Compute $\nabla \L_{\B}(\theta_k) = \frac{1}{m} \sum_{\vect x \in \B} \nabla \ell(\theta_k, \vect x) $
    \State $\theta_{k+1} = \proj(\theta_k - \eta_k \nabla \L_{\B}(\theta_k) + \sqrt{2 \eta_k} \; \gauss{\sigma^2})$
\EndFor
\Ensure $\theta_K$
\end{algorithmic}
\end{algorithm}

We analyse the privacy loss of the \NoisySGD{} algorithm given in Algorithm \ref{algo:noisy_sgd} which implements noisy stochastic gradient descent.%, with additional Gaussian noise on top of the regular stochastic gradient obtained from a batch.

Let $\theta_k$ and $\theta'_k$ denote the parameters at the $k$-th iteration of $\AlgoNoisySGD$ on neighboring datasets $\D$ and $\D'$, respectively.
Stating that $\D$ and $\D'$ are \emph{neighbors} means that they only differ by one $\vect x_{i_0}$, for some index $i_0$. Batch $\B$ is sampled with replacement from $\D$, meaning that a subset of size~$m$ is chosen at random from $\D$ and then replaced at the end of the $k$-th iteration. Hence, $\vect x_{i_0}$ can only appear once in $\B$, with probability $m/n$.
We denote by $\Theta_{t_k}$ and $\Theta'_{t_k}$ the corresponding random variables associated with $\theta_k$ and $\theta'_k$. We abuse notation to also denote their probability density functions by $\Theta_{t_k}$ and $\Theta'_{t_k}$. 
Our objective is to model and analyze the dynamics of differential privacy of this algorithm and to compare it to the ones of the DP-GLD algorithm described by \citet{chourasia2021differential}, which implements full noisy gradient descent.
To this aim, we use the theoretical framework and constructions they provide to analyze the privacy loss of releasing the output $\theta_K$ of the algorithm, assuming private internal states (i.e., $\theta_1, \dots, \theta_{K-1}$).

More precisely, the proof strategy goes as follows: we first model the transition from any step $k$ to the next step $k+1$ in \NoisySGD{} using a diffusion process and derive the evolution equation of the distribution $\Theta_t$ during $k < t < k+1$. We use the theoretical results of \citet{chourasia2021differential} to establish the evolution of the Rényi divergence of two distributions based on neighboring datasets during $k < t < k+1$. Finally, we compute a bound on the global Rényi divergence for the $K$-step \NoisySGD{} process.

% \note{FB: I would put here a summary of the proof strategy. Otherwise people will get lost...}

\subsection{Tracing diffusion for \NoisySGD}

To analyze the privacy loss of \NoisySGD, which is a discrete-time stochastic process, we first interpolate each discrete update from $\theta_k$ to $\theta_{k+1}$ with a piecewise continuously differentiable diffusion process. Given step size $\{\eta_k\}_{k \ge 0}$, variance noise $\sigma^2$ and initial parameter vector $\theta_0$, the $k$-th discrete update in Algorithm \ref{algo:noisy_sgd} is:
\begin{equation}\label{eq:update_eq}
\begin{array}{l}
    \theta_{k+1} = \proj(\theta_k - \eta_k \nabla \L_{\B}(\theta_k) + \sqrt{2 \eta_k \sigma^2} \mathbf{Z}_K)\\[2pt]
    \mathrm{with} \: \mathbf{Z}_k \sim \gauss{} , \B \sampled \D ,
\end{array}
\end{equation}
% \note{FB: what does $\B \sampled \D$ mean? Explain.}
where $\B \sampled \D$ means that batch $\B$ is sampled with replacement from $\D$, and where the loss $\L_{\B}(\theta_k)$ is the estimated empirical loss function over batch $\B$. This discrete jump can be interpolated with the following random process $\{ \Theta_t \}_{t_k \le t \le t_{k+1}}$, where $t_k = \sum_{i=1}^{k} \eta_i$,
\begin{equation}\label{eq:random_process}
\left\{\begin{matrix*}[l]
    \Theta_{t_k} = \theta_k \\
    \Theta_{t_k + \Delta t} - \left ( \Theta_{t_k} - \eta_k \sum_{i = 1, i\neq i_0}^m \dfrac{\nabla \ell(\Theta_{t_k}, \vect x_i)}{m} \right ) \\
    \qquad = - \Delta t \dfrac{\nabla \ell(\Theta_{t_k}, \vect x_{i_0})}{m} + \sqrt{2 \Delta t \, \sigma^2} \mathbf{Z}_k, \: 0 < \Delta t < \eta_k \\[5pt] 
    \Theta_{t_{k+1}} = \proj (\lim_{\Delta t \rightarrow \eta_k} \Theta_{t_k + \Delta t}),\: 0 < \Delta t < \eta_k,
\end{matrix*}\right.
\end{equation}
\noindent where $\mathbf{Z}_k \sim \gauss{}$, $\B$ is sampled with replacement from $\D$, $i_0$ refers to the index in $\B$ of the data item which differs between $\D$ and $\D'$ if it is part of $\B$, otherwise $i_0$ is chosen at random. %$\C$ is a closed convex set that we project onto.

%\noindent Note that this process is not continuous on $t_k^{+}$ and not necessarily in $t_{k+1}^-$ if $\lim_{\Delta t \rightarrow \eta_k} \Theta_{t_k + \Delta t} \notin \C$ but continuity is not needed as long as monotonicity of the Rényi divergence at these points is guaranteed.
We can compute from (\ref{eq:random_process}) $\lim_{\Delta t \rightarrow \eta_k} \Theta_{t_k + \Delta t} = \theta_k - \eta_k \nabla \L_{\B}(\theta_k) + \sqrt{2 \eta_k \sigma^2} \mathbf{Z}_K$. Therefore by the update equation (\ref{eq:update_eq}), we see that the random variable $\Theta_{t_{k+1}} = \proj (\lim_{\Delta t \rightarrow \eta_k} \Theta_{t_k + \Delta t})$ has the same distribution as the parameter $\theta_{k+1}$ at $k+1^{\mathrm{th}}$ step of \NoisySGD.

We now differentiate (\ref{eq:random_process}) over time $t$ for $\{ \Theta_t \}_{t_k < t < t_{k+1}}$ and we derive the following stochastic differential equation.
\begin{equation}\label{eq:SDE}
    d \Theta_t = - \dfrac{\nabla \ell(\Theta_{t_k}, \vect x_{i_0})}{m} dt + \sqrt{2 \sigma^2} d \mathbf{W}_t, 
\end{equation}
\noindent where $d \mathbf{W}_t \sim \sqrt{t} \gauss{} $ describes the Wiener process on $\R^d$, and $i_0$ is chosen as such:
\begin{equation*}%\label{eq:i0}
\left\{\begin{matrix*}[l]
\{ \vect x_{i_0} \} = \{ \vect  x_i, \vect x_i \in \B, \vect x_i \notin \B' \} & \mathrm{if} \: \B \neq \B'\\
\vect x_{i_0} \sampled \B & \mathrm{if} \: \B = \B',
\end{matrix*}\right.
\end{equation*}
\noindent where we assume without loss of generality that two neighboring batches $\B$ and $\B'$ are indexed so as to be equal on all indices but one.

\noindent This shows that $\{ \Theta_t \}_{t_k < t < t_{k+1}}$ is a diffusion process and we repeat the construction for $k = 0, 1, \dots$ to define a piecewise continuous diffusion process $\{ \Theta_t \}_{t \ge 0}$ whose distribution at time $t = t_k$ is consistent with $\theta_k$. We refer to this process as the tracing diffusion for \NoisySGD.

\begin{definition}[Coupled tracing diffusions \cite{chourasia2021differential}]
Let $\Theta_0 = \Theta'_0$ be two identically distributed random variables. We refer to the stochastic processes $\{ \Theta_t \}_{t \ge 0}$ and $\{ \Theta'_t \}_{t \ge 0}$ defined by (\ref{eq:random_process}) as coupled tracing diffusions processes for \NoisySGD{} under loss function $\ell(\theta, \vect x)$ on neighboring datasets $\D$, $\D'$ differing in $i_0$-th data point.
\end{definition}

The Rényi divergence $R_\alpha(\Theta_{t_K} \sep \Theta'_{t_K})$ reflects the Rényi privacy loss of Algorithm \ref{algo:noisy_sgd} with $K$ steps. Conditioned on observing $\theta_k$ and on sampling $\B$ from $\D$, the process $\{ \Theta_t \}_{t_k < t < t_{k+1}}$ is a Langevin diffusion along a constant vector field $\nabla \ell(\theta_k, \vect x_{i_0})$ for duration $\eta_k$. Therefore, the conditional probability distribution $p_{t | t_k}(\theta | \theta_k)$ follows the following Fokker-Planck equation, where the notation $p_{t|t'}(\theta|\theta')$ represents the conditional probability density function $p(\Theta_t = \theta | \Theta_{t'} = \theta')$:
\begin{gather*}
    \dfrac{\partial p_{t | t_k}(\theta | \theta_k)}{\partial t} 
    \!=\! \nabla \!\! \cdot \!\! \left(\! p_{t | t_k}(\theta | \theta_k) \dfrac{\nabla \ell(\theta_k, \vect x_{i_0})}{m} \!\right)\! + \!\sigma^2 \Delta p_{t | t_k}(\theta | \theta_k).
\end{gather*}

\noindent By taking expectations over the distribution $p_{t_k}(\theta_k)$ on both sides, we get the partial differential equation that models the evolution of probability distribution $p_t(\theta)$ in the tracing diffusion.

\begin{lemma}
For the SDE (\ref{eq:SDE}), the equivalent Fokker-Planck equation at time $t_k < t < t_{k+1}$ is
\begin{gather*}
    \dfrac{\partial p_t(\theta)}{\partial t}
    \!=\! \nabla \!\!\cdot\!\! \left(\! p_t(\theta) \!\!\E_{\theta_k \sim p_{t_k | t}}\!\! \left [ \dfrac{\nabla \ell (\theta_k, \vect x_{i_0})}{m} | \theta, \B  \right ] \!\right)\! + \!\sigma^2 \Delta p_t(\theta).
\end{gather*}
\end{lemma}

Using this distribution evolution equation, we model the privacy dynamics in the tracing
diffusion process. This process is similar to Langevin diffusion under conditionally expected loss function $\nabla \L_{\B}(\theta) = \E_{\theta_k \sim p_{t_k | t}} \left [ \dfrac{\nabla \ell (\theta_k, \vect x_{i_0})}{m} \Big| \theta, \B  \right ]$.

\subsection{Privacy erosion in tracing (Langevin) diffusion}

The analysis of the privacy erosion in tracing langevin diffusion is detailed by \citet{chourasia2021differential} but we provide here the key results that we will apply to our algorithm. Privacy erosion refers to the continuous increase of the privacy loss over time as more data is accessed to compute the gradient of the loss in the update equation.

\noindent Consider a Langevin diffusion process modelled through the following Fokker-Planck equation:
\begin{equation*}
     \dfrac{\partial p_t(\theta)}{\partial t}
    = \nabla \cdot (p_t(\theta) \nabla \L_{\B}(\theta) ) + \sigma^2 \Delta p_t(\theta).
\end{equation*}

\begin{definition}[Log-Sobolev inequality \cite{gross1975logarithmic}]
    The distribution of a variable $\Theta$ satisfies the Log-Sobolev Inequality for a constant $c$, or is $c$-LSI, if for all functions $f : \R^d \to \R$ with continuous $\nabla f$ and $\E[f(\Theta)^2] < \infty$, it satisfies
    \begin{gather*}
      \!  \E[f(\Theta)^2 \!\log f(\Theta)^2] \!-\! \E[f(\Theta)^2] \log \E[f(\Theta)^2] \!\le\! \dfrac{2}{c}\! \E[\norm{ \nabla f(\Theta)}^2_2].
    \end{gather*}
\end{definition}    

\begin{lemma}[Dynamics for Rényi privacy loss under c-LSI \cite{chourasia2021differential}]
Assuming that $\Theta$ is $c$-LSI and $S_g$ is the sensitivity of the loss gradient, the dynamics of the Rényi privacy loss can be modelled as such, where $R(\alpha, t)$ represents the $\alpha$ Rényi divergence $R_\alpha(\Theta_t \sep \Theta'_t)$:
\begin{equation}\label{eq:renyi_pde}
    \frac{\partial R(\alpha, t)}{\partial t} \!\le\! \frac{1}{\gamma} \frac{\alpha S_g^2}{4 \sigma^2} \!-\! 2(1 \!-\! \gamma) \sigma^2 c 
    \left [ 
        \frac{ R(\alpha, t)}{\alpha} \!+\! (\alpha \!-\! 1) \frac{\partial R(\alpha, t)}{\partial \alpha}
    \right ],
\end{equation}
for some $\gamma$ which can be arbitrarily fixed.
\end{lemma}

The initial privacy loss satisfies $R(\alpha, 0) = 0$ as $\Theta_0 = \Theta'_0$. The solution $R(\alpha, t)$ for this equation increases with time $t \ge 0$, which models the erosion of Rényi privacy loss in Langevin diffusion. Intuitively, the $c$-LSI condition which provides the negative dependence $\frac{\partial R(\alpha, t)}{\partial t}$ with respect to $R(\alpha, t)$, can be interpreted as a sufficient condition to ensure that the the Rényi privacy loss is bounded, which is further formalized in Theorem \ref{th:rdp_noisysgd_clsi}.

\subsection{Privacy guarantee for \NoisySGD}

We now extend these results to the tracing diffusion for \NoisySGD. In addition, we do not make the assumption that the step size is constant and use a sequence $\{\eta_k\}_{k \ge 0}$ instead. We first bound the gradient sensitivity of the conditionally expected loss $\nabla \L_{\B}(\theta) = \E_{\theta_k \sim p_{t_k | t}} \left [ \frac{\nabla \ell (\theta_k, \vect x_{i_0})}{m} | \theta, \B  \right ]$.

\begin{lemmaE}[Sensitivity][category=lemma:sensitivity]
    Let $\ell(\theta, \vect x)$ be an $L$-Lipschitz loss function on closed convex set $\C$, then for coupled tracing diffusions $\{ \Theta_t \}_{t \ge 0}$ and $\{ \Theta'_t \}_{t \ge 0}$ for \NoisySGD{} on neighboring datasets $\D$ and $\D'$, and noise variance $\sigma^2$,  the gradient sensitivity of conditionally expected loss $\nabla \L_{\B}(\theta) = \E_{\theta_k \sim p_{t_k | t}} \left [ \dfrac{\nabla \ell (\theta_k, \vect x_{i_0})}{m} | \theta, \B  \right ]$ is bounded by:
    \begin{gather*}
        \left \| \E_{\theta_k \sim p_{t_k | t}}\!\! \left [ \dfrac{\nabla \ell (\theta_k, \vect x_{i_0})}{m} | \theta, \B  \! \right] \!- \!\!\!\E_{\theta'_k \sim p'_{t_k | t}}\!\! \left [ \dfrac{\nabla \ell (\theta'_k, \vect x'_{i_0})}{m} | \theta, \B  \!\right ] \right \|_2\!\!\! \le \dfrac{2L}{n},
    \end{gather*}
    where $n$ is the size of the dataset $\D$ and $\D'$, and $m$ is the size of the batch $\B$ and $\B'$.
\end{lemmaE}

\begin{proofE}
We first distinguish on the events $\B = \B'$ and $\B \neq \B'$ and then use the triangle inequality:
\begin{equation*}
\begin{array}{l}
    \Bigg \|  \E_{\theta_k \sim p_{t_k | t}} \left [ \dfrac{\nabla \ell(\theta_k ; \vect{x}_{i_o})}{m} | \theta, \B \right ] \\
    \qquad \qquad \qquad - \; \E_{\theta'_k \sim p'_{t_k | t}} \left [ \dfrac{\nabla \ell(\theta'_k ; \vect{x'}_{i_o})}{m} | \theta, \B\right ] \Bigg \|_2 \\
    = \mathbb{P}[\B \neq \B'] \; \cdot
        \Bigg \|  
            \E \bigg [
                \E\limits_{\theta_k \sim p_{t_k | t}} \left [ \dfrac{\nabla \ell(\theta_k ; \vect{x}_{i_o})}{m} | \theta \right ] \\ 
                \qquad \qquad \qquad - \!\!\!\E\limits_{\theta'_k \sim p'_{t_k | t}} \left [ \dfrac{\nabla \ell(\theta'_k ; \vect{x'}_{i_o})}{m} | \theta \right ] 
            | \, \B \!\neq\! \B' \bigg ]
        \Bigg \|_2 \\
    \le \dfrac{m}{n} \cdot \dfrac{2L}{m}
    \le  \dfrac{2L}{n}
\end{array}
\end{equation*}
\end{proofE}

We now substitute the sensitivity $S_g$ with $\frac{2L}{n}$ in equation (\ref{eq:renyi_pde}) 
modelling the Rényi privacy loss dynamics of tracing diffusion at $t_k \le t < t_{k+1}$ under $c$-LSI condition:

\begin{equation}\label{eq:renyi_pde2}
    \frac{\partial R(\alpha, t)}{\partial t} \!\le\! \frac{1}{\gamma} \frac{\alpha L^2}{n^2 \sigma^2} \!-\! 2(1 \!-\! \gamma) \sigma^2 c 
    \left [ 
        \frac{ R(\alpha, t)}{\alpha} \!+\! (\alpha \!-\! 1) \frac{\partial R(\alpha, t)}{\partial \alpha}
    \right ].
\end{equation}

Following the methodology of \citet{chourasia2021differential}, we solve this PDE under $\gamma = \frac{1}{2}$ and derive the RDP guarantee for the
\NoisySGD{} algorithm.

\begin{theoremE}[RDP for \NoisySGD{} under $c$-LSI][category=th:rdp_noisysgd_clsi]\label{th:rdp_noisysgd_clsi}
    Let $\ell(\theta, \vect x)$ be an $L$-Lipschitz loss function on a closed convex set $\C$. Let $\{ \Theta_t \}_{t \ge 0}$ be the tracing diffusion for $\AlgoNoisySGD$ under loss function $\ell(\theta, \vect x)$ on dataset $\D$. If $\Theta_t$ satisfies $c$-LSI throughout $0 \le t \le \sum_{k=1}^{K} \eta_k$, then $\AlgoNoisySGD$ satisfies $(\alpha, \varepsilon)$-Rényi Differential Privacy for 
    \begin{equation*}
        \varepsilon = \dfrac{2 \alpha L^2}{c n^2 \sigma^4}(1 - e^{- \sigma^2 c \sum_{k=1}^{K} \eta_k}).
    \end{equation*}
\end{theoremE}

\begin{proofE}
At each step $k$, the Rényi privacy follows the evolution equation defined in (\ref{eq:renyi_pde2}) (since it is a tracing diffusion) when $t_k < t < t_{k+1}$, where $t_k = \sum_{i = 1}^{k} \eta_i$:

\begin{equation*}
\resizebox{1\hsize}{!}{$
    \dfrac{\partial R(\alpha, t)}{\partial t} \le \dfrac{1}{\gamma} \dfrac{\alpha L^2}{\sigma^2 n^2} \!-\! 2(1 \!-\! \gamma) \sigma^2 c 
    \left [ 
        \dfrac{ R(\alpha, t)}{\alpha} \!+\! (\alpha \!-\! 1) \dfrac{\partial R(\alpha, t)}{\partial \alpha}
    \right ]
$}
\end{equation*}

Let's $a_1 = 2(1 - \gamma) \sigma^2 c$, $a_2 = \frac{1}{\gamma} \frac{\alpha L^2}{\sigma^2 n^2}$ and 
\begin{equation*}
u(t, y) = 
\left\{\begin{matrix*}[l]
    \frac{R(\alpha, t)}{\alpha} - \frac{a_2}{a_1} & t_k < t < t_{k+1} \\
    \frac{R(\alpha, \lim_{t \to t_k^+} t)}{\alpha} - \frac{a_2}{a_1} & t = t_k
\end{matrix*}\right.
\end{equation*}

We can rewrite the evolution equation as such:
\begin{equation*}
    \frac{\partial u}{\partial t} + a_1 u + a_1 \frac{\partial u}{\partial y} \le 0 \:\:,\: t_k < t < t_{k+1}
\end{equation*}
with initial condition $u(t_k, y) = \frac{R(\alpha, \lim_{t \to t_k^+} t)}{\alpha} - \frac{a_2}{a_1}$.

Let now introduce $\tau = t$ and $z = t - \frac{1}{a_1}y$ and write $v(\tau, z) = u(t, y)$. The equation now writes:
\begin{equation*}
    \frac{\partial v}{\partial \tau} + a_1 v < 0
\end{equation*}
with initial condition $v(t_k, z) = u(t_k, -a_1(z - t_k))$

The solution of this equation is:
\begin{equation*}
    v(\tau, z) \le v(t_k, z)e^{-a_1 (\tau-t_k)}
\end{equation*}
which also writes
\begin{equation*}
    u(t, y) \le u(t_k, y - a_1(t - t_k))e^{-a_1(t-t_k)}
\end{equation*}
or in terms of $R(\alpha, t)$:
\begin{equation*}
    R(\alpha, t) - \frac{a_2}{a_1} \alpha \le (R(\alpha, \lim_{t \to t_k^+}) - \frac{a_2}{a_1} \alpha) e^{-a_1(t-t_k)}
\end{equation*}
for $t_k < t < t_{k+1}$.

Taking the limit $t \to t_{k+1}^-$, we have
\begin{equation*}
    R(\alpha, \lim_{t \to t_{k+1}^-} t) - \frac{a_2}{a_1} \alpha \le (R(\alpha, \lim_{t \to t_k^+} t) - \frac{a_2}{a_1} \alpha) e^{-a_1 \eta_{k+1}}
\end{equation*}

In addition, by the tracing diffusion process described in (\ref{eq:random_process}), we have
\begin{align*}
    \lim_{t \to t_k^+} \Theta_t = \phi(\proj(\lim_{t \to t_k^-} \Theta_t)) \\ 
    \lim_{t \to t_k^+} \Theta_t' = \phi(\proj(\lim_{t \to t_k^-} \Theta_t')) 
\end{align*}
where $\phi(\theta) = \theta - \eta \sum_{i = 1, i\neq i_0}^m \dfrac{\nabla \ell(\theta, \vect x_i)}{m}$ is a mapping which is the same for both processes, as the batches $\B$ for each distribution only differ at most on $\vect x_{i_0}$. Hence, by post-processing of Rényi divergence, we have
\begin{equation*}
     R(\alpha, \lim_{t \to t_{k}^+} t) \le  R(\alpha, \lim_{t \to t_{k}^-} t)
\end{equation*}
Combining the above two inequalities, we derive the following recursive equation
\begin{equation*}
    R(\alpha, \lim_{t \to t_{k+1}^-} t) - \frac{a_2}{a_1} \alpha \le (R(\alpha, \lim_{t \to t_k^-}) - \frac{a_2}{a_1} \alpha) e^{-a_1 \eta_{k+1}}
\end{equation*}
Repeating this step for $k = 0, \dots, K-1$ we have
\begin{equation*}
    R(\alpha, \lim_{t \to t_{K}^-} t) - \frac{a_2}{a_1} \alpha \le (R(\alpha, \lim_{t \to 0^-} t) - \frac{a_2}{a_1} \alpha) e^{-a_1 t_{K}}
\end{equation*}
where $t_K = \sum_{i = 1}^{K} \eta_i$. Because coupled tracing diffusions have the same start parameter, we have $R(\alpha, \lim_{t \to 0^-} t) = 0$. Moreover, since projection is a post-processing mapping we have $R(\alpha, t_K) \le R(\alpha, \lim_{t \to t_{K}^-} t)$. Putting back the values of $a_1$ and $a_2$ we have:
\begin{equation*}
    R(\alpha, t_K) \le \frac{\alpha L^2}{2 \gamma(1- \gamma)c \sigma^4 n^2 }(1 - e^{-2(1- \gamma)\sigma^2 c t_K})
\end{equation*}
Setting $\gamma = \frac{1}{2}$ suffices to prove the Rényi privacy loss bound in the theorem.

\end{proofE}

\emph{Sketch of proof}. We introduce $t_k = \sum_{i = 1}^{k} \eta_i$ for $k \ge 0$.
The idea of the proof is to bound $R(\alpha, t_K)$, the Rényi divergence after $K$ updates in \NoisySGD, with a function of $R(\alpha, t_0)$.
This is first done by considering the $k$-th update, and proving the following equation for some constants $a_1$, $a_2$
\begin{equation*}
    R(\alpha, \lim_{t \to t_{k+1}^-} t) - \frac{a_2}{a_1} \alpha \le (R(\alpha, \lim_{t \to t_k^+} t) - \frac{a_2}{a_1} \alpha) e^{-a_1 \eta_{k+1}},
\end{equation*}
together with 
\begin{equation*}
     R(\alpha, \lim_{t \to t_{k}^+} t) \le  R(\alpha, \lim_{t \to t_{k}^-} t).
\end{equation*}
This allows to bound $R(\alpha, \lim_{t \to t_{k+1}^-} t)$ with a function of $R(\alpha, \lim_{t \to t_{k}^-} t)$ and the final bound follows by recursivity, by noting that $R(\alpha, t_0^-) = 0$ since coupled tracing diffusions have the same start parameter.

This theorem guarantees that under the c-LSI condition, the privacy loss converges during the noisy SGD process if $\lim_{K \to \infty} \sum_{k=1}^K \eta_k = \infty $.

In particular, the case where the step size is constant is straightforward:

\begin{corollary}[RDP for \NoisySGD{} under $c$-LSI with constant step-size]\label{coro:rdp_noisysgd_clsi}
    Let $\Theta_t$ be defined as in Theorem \ref{th:rdp_noisysgd_clsi}. If $\AlgoNoisySGD$ has constant step size $\eta$ and if $\Theta_t$ satisfies $c$-LSI throughout $0 \le t \le \eta K$, then $\AlgoNoisySGD$ satisfies $(\alpha, \varepsilon)$ Rényi Differential Privacy for 
    \begin{equation*}
        \varepsilon = \dfrac{2 \alpha L^2}{c n^2 \sigma^4}(1 - e^{- c \sigma^2 \eta K}).
    \end{equation*}
\end{corollary}

In addition, \citet{chourasia2021differential} show that the $c$-LSI condition is satisfied in DP-GLD with constant step size, for loss functions that are Lipschitz, strongly convex and smooth, with appropriate conditions on the algorithm parameters and initialization. We derive an equivalent lemma for DP-SGLD with varying step size.

\begin{lemma}[LSI for \NoisySGD{}]
    If loss function $\ell(\theta, \vect x)$ is $\lambda$-strongly convex and $\beta$-smooth over a closed convex set $\C$, then the coupled tracing diffusion processes $\{ \Theta_t \}_{t \ge 0}$ and $\{ \Theta'_t \}_{t \ge 0}$ for \NoisySGD{} with step size $\{ \eta_k \}_{k \ge 0}$ satisfying $\eta_k < \frac{1}{\beta}$ for $k \ge 0$, and with initial distribution $\Theta_0 \sim \proj(\gauss{\frac{2\sigma^2}{\lambda}})$, satisfy $c$-LSI for $t \ge 0$ with $c = \frac{\lambda}{2 \sigma^2}$.
\end{lemma}

The proof of this lemma is exactly the same as Lemma 5 of \citet{chourasia2021differential}, where $n$ needs to be replaced with the batch size $m$ and $\eta$ with $\eta_k$.

We immediately derive the following bound on the Rényi privacy loss for \NoisySGD.

\begin{theorem}[Privacy guarantee for \NoisySGD]\label{th:privacy_SGLD}
    Let $\ell(\theta, \vect x)$ be an $L$-Lipschitz, $\lambda$-strongly convex and $\beta$-smooth loss function on closed convex set $\C$, then $\AlgoNoisySGD$ with start parameter $\theta_0 \sim \proj(\gauss{\frac{2\sigma^2}{\lambda}})$ and step-size $\eta < \frac{1}{\beta}$ satisfies $(\alpha, \varepsilon)$-Rényi differential privacy with 
    \begin{equation*}
        \varepsilon = \dfrac{4 \alpha L^2}{\lambda n^2 \sigma^2}(1 - e^{- \frac{\lambda}{2} \sum_{k=1}^K \eta_k }).
    \end{equation*}
\end{theorem}

The case where the step size is constant follows:

\begin{corollary}[Privacy Guarantee for \NoisySGD{} with constant step-size]\label{coro:privacy_SGLD_constant}
    With $\ell(\theta, \vect x)$ and $\AlgoNoisySGD$ defined as in Theorem \ref{th:privacy_SGLD}, and with constant step size $\eta < \frac{1}{\beta}$, $\AlgoNoisySGD$ satisfies $(\alpha, \varepsilon)$-Rényi Differential Privacy with 
    \begin{equation*}
        \varepsilon = \dfrac{4 \alpha L^2}{\lambda n^2 \sigma^2}(1 - e^{- \lambda \eta K / 2}).
    \end{equation*}
\end{corollary}

We also provide the case where the step size is defined as $\smash{\eta_k = \frac{1}{2 \beta + \lambda k / 2}}$, which is further analyzed in the next section.

\begin{corollary}[Privacy Guarantee for \NoisySGD{} with decreasing step-size]\label{coro:privacy_SGLD_decreasing}
    With $\ell(\theta, \vect x)$ and $\AlgoNoisySGD$ defined as in Theorem \ref{th:privacy_SGLD}, and with step size $\eta_k = \frac{1}{2 \beta + \lambda k / 2}$, $\AlgoNoisySGD$ satisfies $(\alpha, \varepsilon)$-Rényi Differential Privacy with 
    \begin{equation*}
        \varepsilon = \dfrac{4 \alpha L^2}{\lambda \sigma^2 n^2}(1 - e^{- \log(1 + \frac{\lambda K}{4 \beta})}) =  \dfrac{4 \alpha L^2}{\lambda n^2 \sigma^2} \frac{\lambda K}{4 \beta + \lambda K}.
    \end{equation*}
\end{corollary}

\paragraph{Discussion.} Let us consider that $\eta$ in Corollary \ref{coro:privacy_SGLD_constant} is set as $\eta = \frac{1}{2 \beta}$, so that it can be compared to Corollary \ref{coro:privacy_SGLD_decreasing} and also matches the maximum upper bound for which we derive utility guarantees is the next section. 
In the regime where $K$ is small (compared to $\frac{\beta}{\lambda}$), both corollaries have equivalent bounds on $\varepsilon$. Indeed, in the fixed step size setting we have with $\smash{\eta = \frac{1}{2\beta}}$:
\begin{equation*}
    \varepsilon = \frac{4 \alpha L^2}{\lambda n^2 \sigma^2}(1 - e^{- \lambda K / {4 \beta}}) \sim_{K \ll \frac{\beta}{\lambda} } \frac{\alpha L^2 K}{\beta n^2 \sigma^2},
\end{equation*}
while in the decreasing step size setting we have
\begin{equation*}
    \varepsilon = \dfrac{4 \alpha L^2}{\lambda n^2 \sigma^2} \frac{\lambda K}{4 \beta + \lambda K}
    \sim_{K \ll \frac{\beta}{\lambda} } \frac{\alpha L^2 K}{\beta n^2 \sigma^2}.
\end{equation*}
In particular, $\varepsilon$ reaches the baseline composition analysis from \citet{abadi2016deep} up to a factor 2 : $\varepsilon' = \frac{\alpha L^2}{n^2 \sigma^2} \cdot \eta K = \frac{\alpha L^2 K}{2 \beta n^2 \sigma^2}$. 

In the regime where $K$ is sufficiently large (compared to $\frac{\beta}{\lambda}$), both the fixed $\eta$ and decreasing $\eta_k$ settings also reach the same bound on $\varepsilon$, equal to
\begin{equation*}
    \varepsilon \sim_{K \gg \frac{\beta}{\lambda} } \frac{4 \alpha L^2}{\lambda n^2 \sigma^2}.
\end{equation*}

As a side note, we notice that we can consider the unconstrained regularized version of the problem: $\widetilde{\L_\D}(\theta) = \L_\D(\theta) + \frac{\lambda}{2} \norm{\theta}^2_2$ and derive equivalent properties. In this scenario, we no longer need the strong convexity assumption on $\L_\D(\theta)$. In addition we can use the optimality of $\theta^*$ for $\widetilde{\L_\D}$ to derive two equations :
\begin{align*}
    & \L_\D(\theta^*) + \frac{\lambda}{2} \norm{\theta^*}^2_2 \le \L_\D(0), \ \ \ \ \  c\nabla \L_\D(\theta^*) + \lambda \theta^* = 0.
\end{align*}
Each of these provides a bound on $\norm{\theta^*}_2$, by using respectively the positivity of $\L_\D$ and the lipschitzness of $\L_\D$,
\begin{equation*}
    \norm{\theta^*}_2 \le \Big(\frac{2 \L_\D(0)}{\lambda}\Big)^{1/2}, \quad
    \norm{\theta^*}_2 \le \frac{L}{\lambda},
\end{equation*}
which can be used as bounds for the radius of the convex $\C$ we project onto, so that we still actually end up solving the unconstrained problem.

% \note{FB: where do we need exactly to project? Wouldn't it be simpler if you consider explicitly a regularized problem with $\lambda \| \theta\|^2$ so that we can have both Lipschitz-continuity and strong-convexity, which are incompatible without a compact $\C$? Let's discuss.}

\section{Utility analysis for noisy stochastic gradient descent}\label{sec:utility_SGLD}

Differential privacy is known for setting a trade-off between privacy and utility. To assess the utility of the noisy stochastic gradient descent algorithm $\AlgoNoisySGD$, we measure two quantities, the worst case excess empirical risk
\begin{equation*}
    \max_{\D \in \X^n} \E[\L_{\D}(\theta_K) - \L_\D(\theta^*)],
\end{equation*}
and the worst case average empirical risk
\begin{equation}\label{eq:avg_emp_rik}
    \max_{\D \in \X^n} \E[\frac{1}{K}\sum_{k=1}^K \L_{\D}(\theta_k) - \L_\D(\theta^*)],
\end{equation}
where $\theta_K$ is the output of the randomized algorithm $\AlgoNoisySGD$ on $\D$ during $K$ iterations, $\theta^*$ is the solution to the standard non-noisy GD algorithm and the expectation is taken over the randomness of the algorithm. 

\subsection{Fixed step size $\eta$}

We propose a bound on the worst case excess empirical risk when the learning rate is fixed and satisfies $\eta < \frac{1}{2 \beta}$.

\begin{lemmaE}[Empirical risk for smooth and strongly convex loss][category=lemma:emp_risk]
\label{lemma:emp_risk}

    Let $\ell(\theta, \vect x)$ be an $L$-Lipschitz, $\lambda$-strongly convex and $\beta$-smooth loss function on closed convex set $\C$, $\AlgoNoisySGD$ be parameterized with step-size $\eta < \frac{1}{2 \beta}$ and start parameter $\smash{\theta_0 \sim \proj(\gauss{\frac{2\sigma^2}{\lambda}})}$, then the empirical risk of $\AlgoNoisySGD$ is bounded by
    \begin{multline}
        \E[ \L_{\D}(\theta_K\!) \!-\! \L_{\D}(\theta^*\!)]
        %\le \dfrac{2 \beta L^2}{\lambda^2}
        \!\le\! \dfrac{\beta}{2}\! \E[\norm{\theta_0 \!-\! \theta^*\!}_2^2]
        e^{-\lambda \eta K} \\
        \!\!+\! \dfrac{\beta \eta \xi^2 }{2 \lambda} \!+\! \dfrac{\beta d \sigma^2}{\lambda},
    \end{multline}
where $\theta^*$ is the minimizer of $\L_{\D}(\theta)$ in $\C$ and $\xi^2 = \E[\norm{\nabla \L_{\B}(\theta^*)}_2^2]$.
\end{lemmaE}

\begin{proofE}

%We analyse first the excess risk due to adding the noise and then the one due to using SGD instead of GD, by dividing the excess empirical risk as such:

%\[
%    \E[ \L_{\D}(\theta_K) -  \L_{\D}(\hat \theta)]  = \E[ \L_{\BK}(\theta_K) -  \L_{\BK}(\theta^*)] + %\E[\L_{\BK}(\theta^*) - \L_{\D}(\hat \theta)]
%\]

% Note that $\L_{\B}(\cdot)$ is an unbiased estimate of $ \L_{\D}(\cdot)$. Therefore, we have $\E[ \L_{\D}(\theta_K)] = \E[ \L_{\B}(\theta_K)]$. \\

%Let's first note that $\L_{\B}(\cdot)$ is an unbiased estimate of $ \L_{\D}(\cdot)$. Therefore, we have

%\[
%    \E[ \L_{\D}(\theta_K) -  \L_{\D}(\theta^*)] = \E[ \L_{\B}(\theta_K) - \L_{\B}(\theta^*)] 
%\]

Let's recall the noisy SGD update equation:
\begin{equation*}%\label{eq:noisy_sgd_update}
    \theta_{k+1} = \proj(\theta_k - \eta \nabla \L_{\B}(\theta_k) + \sqrt{2 \eta \sigma^2} \gauss{})
\end{equation*}

From the definitions of $\proj(\cdot)$ and $\theta^*$, we have:
\begin{equation*}%\label{eq:optimal_projection}
    % \proj(\theta^* - \eta \nabla \L_{\D}(\theta^*)) = \theta^*
    \proj(\theta^*) = \theta^*
\end{equation*}

Combining this facts and using the contractivity of the projection $\proj(\cdot)$, we derive:
$$\begin{array}{l}
    \norm{\theta_{k+1} - \theta^*}_2^2 \\
    \; = \norm{\proj(\theta_k \!-\! \eta \nabla \L_{\B}\!(\theta_k) \!+\! \sqrt{2 \eta \sigma^2} \gauss{}) \!-\! \theta^*}_2^2 \\ % \mathrm{using\:(\ref{eq:noisy_sgd_update}) }\\
    \; = \norm{\proj(\theta_k \!-\! \eta \nabla \L_{\B}\!(\theta_k) \!+\! \sqrt{2 \eta \sigma^2} \gauss{}) \!-\! \proj(\theta^*)}_2^2 \\ % \mathrm{using\:(\ref{eq:optimal_projection}) } \\
    \; \le \norm{\theta_k - \eta \nabla \L_{\B}(\theta_k) + \sqrt{2 \eta \sigma^2} \gauss{} - \theta^*}_2^2 \\
    \; = \norm{\theta_k - \theta^* - \eta \nabla \L_{\B}(\theta_k) + \sqrt{2 \eta \sigma^2} \gauss{} }_2^2 \\
    \; = \norm{\theta_k - \theta^*}_2^2 + \eta^2 \norm{\nabla \L_{\B}(\theta_k)}_2^2 + 2 \eta \sigma^ 2 \norm{\gauss{}}_2^2\\
    \quad +2 \langle \theta_k - \theta^*, \sqrt{2 \eta \sigma^2} \gauss{} \rangle \\
    \quad - 2 \eta \langle \nabla \L_{\B}(\theta_k), \sqrt{2 \eta \sigma^2} \gauss{} \rangle \\
    \quad  - 2 \eta \langle \theta_k - \theta^*, \nabla \L_{\B}(\theta_k) \rangle \\
     
\end{array}$$
% where the green terms disappear in the expectation, and will hence be omitted from now.\\

We now take the expectation with respect to the random variable $\B$ sampled from $\D$, and we note that $\E_\D[\L_{\B}(\theta_k)] = \L_{\D}(\theta_k)$ since $\L_{\B}$ is an unbiased estimate of $\L_{\D}$.
$$\begin{array}{l}
    \E_\D[\norm{\theta_{k+1} - \theta^*}_2^2] \\
    \; = \norm{\theta_k - \theta^*}_2^2 + \eta^2 \E_\D[\norm{\nabla \L_{\B}(\theta_k)}_2^2] \\
    \quad + 2 \eta \sigma^ 2 \E_\D[\norm{\gauss{}}_2^2]\\
    \quad +2 \langle \theta_k - \theta^*, \sqrt{2 \eta \sigma^2} \gauss{} \rangle \\
    \quad - 2 \eta \langle \nabla \L_{\D}(\theta_k), \sqrt{2 \eta \sigma^2} \gauss{} \rangle \\
    \quad  - 2 \eta \langle \theta_k - \theta^*, \nabla \L_{\D}(\theta_k) \rangle \\
\end{array}$$

Using the classic inequality $(a + b)^2 \le 2(a^2 + b^2)$, we derive:
$$\begin{array}{l}
\E_\D[\norm{\nabla \L_{\B}(\theta_k)}_2^2] \\
\; = \E_\D[\norm{\nabla \L_{\B}(\theta_k) - \nabla \L_{\B}(\theta^*) + \nabla \L_{\B}(\theta_*) }_2^2] \\
\; \le\! 2 \E_\D[\norm{\nabla \L_{\B}\!(\theta_k) \!-\! \nabla \L_{\B}\!(\theta^*)}_2^2] \!+\! 2 \E_\D[\norm{\nabla \L_{\B}\!(\theta^*) }_2^2] \\
%& \:\:\: + \E_\D[ \langle \nabla \L_{\B}(\theta_k) - \nabla \L_{\B}(\theta^*), \nabla \L_{\B}(\theta^*) \rangle] \textcolor{blue}{\:\} = 0 ???? TODO}
\end{array}$$

Furthermore, by $\beta$-smoothness of $\nabla \L_{\B}$, we can use the property of co-coercivity of the gradients:
\begin{equation}\label{eq:cocoercitivity}
    \norm{\nabla \L_{\B}\!(\theta_k) \!-\! \nabla \L_{\B}\!(\theta^*\!)}_2^2 \!\le\! \beta \langle \theta_k \!-\! \theta^*\!, \!\nabla \L_{\B}\!(\theta_k) \!-\! \nabla \L_{\B}\!(\theta^*\!) \rangle
\end{equation}

Taking expectation over $\B$ and using optimality of $\theta^*$:
\begin{equation}\label{eq:expected_coco}
\begin{array}{l}
    \E_{\D}[\norm{\nabla \L_{\B}(\theta_k) - \nabla \L_{\B}(\theta^*)}_2^2]  \\
    \; \le \beta \langle \theta_k - \theta^*, \nabla \L_{\D}(\theta_k) - \nabla \L_{\D}(\theta^*) \rangle \\
    \; = \beta \langle \theta_k - \theta^*, \nabla \L_{\D}(\theta_k) \rangle
\end{array}
\end{equation}

Combining these elements, we derive:
$$\begin{array}{l}
    \E_\D[\norm{\theta_{k+1} - \theta^*}_2^2] \\
    \; = \norm{\theta_k - \theta^*}_2^2 + 2 \eta^2 \E_\D[\norm{\nabla \L_{\B}(\theta_k) - \nabla \L_{\B}(\theta^*)}_2^2] \\
    \quad + 2 \eta^2 \E_\D[\norm{\nabla \L_{\B}(\theta^*)}_2^2] + 2 \eta \sigma^ 2 \E_\D[\norm{\gauss{}}_2^2] \\
    \quad +2 \langle \theta_k - \theta^*, \sqrt{2 \eta \sigma^2} \gauss{} \rangle \\
    \quad - 2 \eta \langle \nabla \L_{\D}(\theta_k), \sqrt{2 \eta \sigma^2} \gauss{} \rangle \\
    \quad  - 2 \eta \langle \theta_k - \theta^*, \nabla \L_{\D}(\theta_k) \rangle \\
    \; \le \norm{\theta_k - \theta^*}_2^2 + 2 \eta (\eta \beta - 1) \langle \theta_k - \theta^*, \nabla \L_{\D}(\theta_k) \rangle \\
    \quad + 2 \eta^2 \E_\D[\norm{\nabla \L_{\B}(\theta^*) }_2^2] + 2 \eta \sigma^ 2 \E_\D[\norm{\gauss{}}_2^2]\\
    \quad +2 \langle \theta_k - \theta^*, \sqrt{2 \eta \sigma^2} \gauss{} \rangle \\
    \quad - 2 \eta \langle \nabla \L_{\D}(\theta_k), \sqrt{2 \eta \sigma^2} \gauss{} \rangle \\
    \quad \mathrm{\:using\:(\ref{eq:cocoercitivity})\:and\:}(\ref{eq:expected_coco})
\end{array}$$

Let's assume that $\eta \le \frac{1}{2 \beta}$. Therefore, $2 \eta (\eta \beta - 1) \le - \eta$. 
By optimality of $\theta^*$ and strong convexity arguments on $\L_\D$, we deduce the following inequality:
$$\begin{array}{ll}
\!\!- \langle \theta_k \!-\! \theta^*\!, \nabla \L_{\D}(\theta_k) \rangle \!\!\!\!\!
& \le \L_\D(\theta_k) \!-\! \L_\D(\theta^*\!) \!+\! \dfrac{\lambda}{2}\! \norm{\theta_k \!-\! \theta^*}_2^2 \\
& \le \lambda \norm{\theta_k - \theta^*}_2^2
\end{array}$$

Plugging this together, we have:
$$\begin{array}{l}
    \E_\D[\norm{\theta_{k+1} - \theta^*}_2^2] \\
    \; \le (1 - \eta \lambda) \norm{\theta_k - \theta^*}_2^2 + \eta^2 \E_\D[\norm{\nabla \L_{\B}(\theta^*)}_2^2] \\
    \quad + 2 \eta \sigma^ 2 \E_\D[\norm{\gauss{}}_2^2] \\
    \quad +2 \langle \theta_k - \theta^*, \sqrt{2 \eta \sigma^2} \gauss{} \rangle \\
    \quad - 2 \eta \langle \nabla \L_{\D}(\theta_k), \sqrt{2 \eta \sigma^2} \gauss{} \rangle \\
\end{array}$$

We take the expectation again:
\begin{equation} \label{eq:final_theta_diff}
    \E[\norm{\theta_{k+1} - \theta^*}_2^2]
    \le (1 - \eta \lambda) \E[\norm{\theta_k - \theta^*}_2^2] \qquad
\end{equation}
\vspace{-8mm}
$$\qquad \qquad \qquad + \eta^2 \E[\norm{\nabla \L_{\B}(\theta^*)}_2^2] + 2 \eta d \sigma^ 2$$
\\
Now, by $\beta$-smoothness of $\L_{\D}(\cdot)$, we have:
\[
    \L_{\D}(\theta_K) - \L_{\D}(\theta^*) \le \langle \nabla \L_{\D}(\theta^*), \theta_K - \theta^* \rangle + \dfrac{\beta}{2} \norm{\theta_K - \theta^*}_2^2
 \]

Second, by optimality of $\theta^*$ in $\C$ and the fact that $\theta_K \in \C$, we have 
\begin{equation*}
    \langle \nabla \L_{\D}(\theta^*), \theta_K - \theta^* \rangle = 0
\end{equation*}

Combining these two we get:
\[
    \E[ \L_{\D}(\theta_K) - \L_{\D}(\theta^*)] \le \dfrac{\beta}{2} \E[\norm{\theta_K - \theta^*}_2^2]
\]

Using the recursive equation (\ref{eq:final_theta_diff}) repeatedly for $k = 0, \dots , K - 1$, we have:
$$\begin{array}{l}
    \E[ \L_{\D}(\theta_K) - \L_{\D}(\theta^*)] \\
    \; \le \dfrac{\beta}{2} (1 - \eta \lambda)^K  \E[\norm{\theta_0 - \theta^*}_2^2]  \\
    \qquad\qquad\qquad\quad + \dfrac{\beta}{2} (\eta^2 \xi^2 + 2 \eta d \sigma^2) \sum_{k= 0}^{K-1} (1 - \eta \lambda)^k \\
    \; \le \dfrac{\beta}{2} e^{-\lambda \eta K} \E[\norm{\theta_0 - \theta^*}_2^2]
    + \dfrac{\beta \eta \xi^2 }{2 \lambda} + \dfrac{\beta d \sigma^2}{\lambda} \\[7pt]
    \mathrm{where}\: \xi^2 = \E[\norm{\nabla \L_{\B}(\theta^*)}_2^2]
\end{array}$$

\end{proofE}

\textit{Sketch of proof.} First, we recursively bound $\smash{\norm{\theta_{k+1} - \theta^*}_2^2}$ as a function of $\smash{\norm{\theta_{k} - \theta^*}_2^2}$, using the definition of $\theta_{k+1}$. Then, we take the expectation with respect to $\B$ and use co-coercivity of the gradients, and take the expectation again to derive a recursive relationship between $\smash{\E[\norm{\theta_{k+1} - \theta^*}_2^2]}$ and $\smash{\E[\norm{\theta_{k} - \theta^*}_2^2]}$. Last, we express the empirical risk $\smash{\E[ \L_{\D}(\theta_K) - \L_{\D}(\theta^*)]}$ as a function of $\smash{\E[\norm{\theta_{K} - \theta^*}_2^2]}$. 

This shows that under Lipschitz smooth strongly convex loss function, the empirical risk of $\AlgoNoisySGD$ decreases as the iterations increase, and reaches a constant factor which is the sum of a term directly related to the variance of the noise $\sigma^2$ added at each iteration and another term which represents the error due to the SGD process, which decreases with the learning rate $\eta$.

\begin{lemma}[Empirical risk for smooth and strongly convex loss, independent of $\theta$]\label{lemma:emp_risk_wo_theta}
    Let $\ell$ and $\AlgoNoisySGD$ be defined as in Lemma \ref{lemma:emp_risk}, then the empirical risk is bounded by
    \begin{equation*}
        \E[ \L_{\D}(\theta_K\!) \!-\! \L_{\D}(\theta^*\!)]
        \le \dfrac{2 \beta L^2}{\lambda^2}
        e^{-\lambda \eta K}
        \!\!+\! \dfrac{\beta \eta \xi^2 }{2 \lambda} \!+\! \dfrac{\beta d \sigma^2}{\lambda},
    \end{equation*}
where $\xi^2 = \E[\norm{\nabla \L_{\B}(\theta^*)}_2^2]$.
\end{lemma}
\begin{proof}
Since we have $\norm{\C} \le \frac{2 L}{\lambda}$, we can bound $\E[\norm{\theta_0 - \theta^*}_2^2] \le \frac{4L^2}{\lambda^2}$, as $\theta_0, \theta^* \in \C$.
\end{proof}

Combining now Lemma \ref{lemma:emp_risk_wo_theta} and Theorem \ref{th:privacy_SGLD}, we derive the utility of $\AlgoNoisySGD$ under $(\alpha, \varepsilon)$-Rényi differential privacy.

\begin{theoremE}[Utility bound for $(\alpha, \varepsilon)$-Rényi differential privacy][category=th:utility_renyi]\label{th:utility_renyi}
    Let $\ell(\theta, \vect x)$ be an $L$-Lipschitz, $\lambda$-strongly convex and $\beta$-smooth loss function on closed convex set $\C$, then $\AlgoNoisySGD$ with start parameter $\smash{\theta_0 \sim \proj(\gauss{\frac{2\sigma^2}{\lambda}})}$ and constant step-size $\eta = \frac{1}{2 \beta}$, satisfies $(\alpha, \varepsilon)$ Rényi differential privacy and
    \begin{equation*}
        \E[ \L_{\D}(\theta_K) - \L_{\D}(\theta^*)] = O\left(\dfrac{ \alpha \beta d L^2}{\varepsilon \lambda^2 n^2} \right) + \dfrac{\xi^2}{4 \lambda},
    \end{equation*}
    where $\sigma^2$ and $K$ are set as such:
    \begin{equation*}
        \sigma^2 = \dfrac{4 \alpha L^2}{\varepsilon \lambda n^2}, \quad
        K = \dfrac{2 \beta}{\lambda} \log \! \left( \dfrac{\varepsilon n^2}{\alpha d} \right).
    \end{equation*}
\end{theoremE}

\begin{proofE}
First, as $\theta_0, \theta^* \in \C$ and $\norm{C}_2 \le \frac{2L}{\lambda}$, we have $\E[\norm{\theta_0 - \theta^*}_2^2] \le \frac{4L^2}{\lambda^2}$.

Then, we plug the values of $\sigma^2$ and $K$ in Lemma \ref{lemma:emp_risk}:
\[
\begin{array}{ll}
    \E[ \L_{\D}(\theta_K\!) \!-\! \L_{\D}(\theta^*\!)]\!\!\!\!
    & \!\le\! \dfrac{\beta}{2}\! \dfrac{4L^2}{\lambda^2}
      e^{-\lambda \eta K} \\
    & \quad + \dfrac{\beta d \sigma^2}{\lambda} \!+\!  \dfrac{\beta \eta \xi^2 }{2 \lambda} \\ 
    & \!\le\! \dfrac{2 \beta L^2}{\lambda^2}\! 
       \dfrac{\alpha d}{\varepsilon n^2}  \\ 
    & \quad + \dfrac{\beta d}{\lambda} \dfrac{4 \alpha L^2}{\lambda \varepsilon n^2} \!+\!  \dfrac{\beta \eta \xi^2 }{2 \lambda} \\
    & \le \dfrac{6 \alpha \beta d L^2}{\varepsilon \lambda^2 n^2}  + \dfrac{\xi^2}{4 \lambda} \:\mathrm{with}\: \eta \!=\! \dfrac{1}{2 \beta}
     
\end{array}
\]

\end{proofE}

\begin{theoremE}[Utility bound for $(\epsilon, \delta)$-differential privacy][category=th:utility_dwork]
  With the same conditions as is Theorem \ref{th:utility_renyi}, for $\epsilon \le 2 \log(1/\delta)$ and $\delta > 0$, $\AlgoNoisySGD$ satisfies $(\epsilon, \delta)$ differential privacy and
    \begin{equation*}
        \E[ \L_{\D}(\theta_K) - \L_{\D}(\theta^*)] = O\left(\dfrac{ \beta d L^2 \log(1/\delta)}{\epsilon^2 \lambda^2 n^2} \right) + \dfrac{\xi^2}{4 \lambda},
    \end{equation*}
    where $\sigma^2$ and $K$ are set as such:
    \begin{align*}
        \sigma^2 = \dfrac{8 L^2 (\epsilon + 2 \log(1/\delta)}{\epsilon^2 \lambda n^2} \\
        K = \dfrac{2 \beta}{\lambda} \log \! \left( \dfrac{\epsilon^2 n^2}{4 \log(1/\delta) d} \right).
    \end{align*}
\end{theoremE}

\begin{proofE}
By setting $\varepsilon = \frac{\epsilon}{2}$, we derive from Proposition \ref{prop:renyi_to_dp} :
\begin{equation*}
    \alpha = 1 + \frac{2}{\epsilon} \log(1 / \delta)
\end{equation*}.

We use this to rewrite the results from Theorem \ref{th:utility_renyi}:
\[
\begin{array}{ll}
    \E[ \L_{\D}(\theta_K\!)\!- \!\L_{\D}(\theta^*\!)]\!\! & \le \dfrac{6 \alpha \beta d L^2}{\varepsilon \lambda^2 n^2} + \dfrac{\xi^2}{4 \lambda} \\
    & = \dfrac{6 \beta d L^2}{\lambda^2 n^2} \dfrac{1 + \frac{2}{\epsilon} \log(1 / \delta)}{\frac{\epsilon}{2}} + \dfrac{\xi^2}{4 \lambda} \\ 
    & = \dfrac{6 \beta d L^2}{\lambda^2 n^2} \dfrac{2 \epsilon + 4 \log(1 / \delta)}{\epsilon^2} + \dfrac{\xi^2}{4 \lambda} \\
    & \le \dfrac{6 \beta d L^2}{\lambda^2 n^2} \dfrac{8 \log(1 / \delta)}{\epsilon^2} + \dfrac{\xi^2}{4 \lambda}\\
    & \quad \mathrm{using} \: \epsilon \le 2 \log(1/\delta)
\end{array}
\]
Similarly,
\[
\begin{array}{ll}
    \sigma^2\! &= \dfrac{4 L^2 \alpha}{\lambda n^2 \varepsilon} \\
    &= \dfrac{4 L^2}{\lambda n^2}  \dfrac{1 + \frac{2}{\epsilon} \log(1 / \delta)}{\frac{\epsilon}{2}} \\
    &= \dfrac{8 L^2 (\epsilon + 2 \log(1/\delta)}{\epsilon^2 \lambda n^2}
\end{array}
\]
\[
\begin{array}{ll}
    K\! &= \dfrac{2 \beta}{\lambda} \log \! \left( \dfrac{n^2}{ d} \dfrac{\varepsilon}{\alpha } \right)\\
    &= \dfrac{2 \beta}{\lambda} \log \! \left( \dfrac{n^2}{ d} \dfrac{\frac{\epsilon}{2}}{1 + \frac{2}{\epsilon} \log(1 / \delta) } \right) \\
    &= \dfrac{2 \beta}{\lambda} \log \! \left( \dfrac{n^2}{ d} \dfrac{\epsilon^2}{ 2 \epsilon + 4 \log(1 / \delta) } \right) \\
    &\le \dfrac{2 \beta}{\lambda} \log \! \left( \dfrac{ \epsilon^2 n^2}{ 4 \log(1 / \delta) d} \right) 
\end{array}
\]
\end{proofE}

As a side note, arguments of the proof of Lemma \ref{lemma:emp_risk} (like the co-coercivity of the gradients) can be reused to improve the upper bound on $\eta$ from $\frac{\lambda}{2 \beta^2}$ to $\frac{1}{2 \beta}$ in Theorem 4 of \citet{chourasia2021differential} on the utility of DP-GLD. We provide experimental evidence in the next section that this factor $\frac{\lambda}{\beta}$ is non-negligible.

% Proofs of both theorems are given in Appendix \ref{app:utility_SGLD:fixed_step}.

\subsection{Decreasing step size $\eta_k$}

To remove the $\frac{\xi^2}{4 \lambda}$ term which is due to using stochastic gradient descent, we follow the approach from \citet{NIPS2012_905056c1} and propose to bound the worst case average empirical risk (\ref{eq:avg_emp_rik}) when the step size is decreasing and follows:
\vspace{-0.1cm}
\begin{equation*}%\label{eq:eta_k}
    \eta_k = \dfrac{1}{2 \beta + \frac{\lambda k}{2}} \, , \, k \ge 0.
\end{equation*}

%In this subsection, we denote the learning rate with $\eta_k$ and the update equation from algorithm \ref{algo:noisy_sgd} should now read:
%\[
%    \theta_{k+1} = \proj(\theta_k - \eta_{k+1} \nabla \L_{\B}(\theta_k) + \sqrt{2 %\eta_{k+1}} \; \gauss{\sigma^2})
% \]

\vspace{-0.2cm}
\begin{lemmaE}[Empirical risk for smooth and strongly convex loss with decreasing learning rate][category=lemma:emp_risk_decreasing_eta]\label{lemma:emp_risk_decreasing_eta}
    Let $\ell(\theta, \vect x)$ be an $L$-Lipschitz, $\lambda$-strongly convex and $\beta$-smooth loss function on closed convex set $\C$, $\AlgoNoisySGD$ be parameterized with \emph{decreasing} step-size $\smash{\eta_k = \frac{1}{2 \beta + \lambda k / 2}}$ and start parameter $\smash{\theta_0 \sim \proj(\gauss{\frac{2\sigma^2}{\lambda}})}$, then the average empirical risk of $\AlgoNoisySGD$ is bounded by
    \begin{multline}
        \E[\frac{1}{K}\!\!\sum_{k=1}^K \!\L_{\D}(\theta_k\!) \!-\! \L_\D(\theta^*\!)]
        %\le \dfrac{2 \beta L^2}{\lambda^2}
        \!\le\! \dfrac{2 \beta}{K} \E[\norm{\theta_{0} \!-\! \theta^*\!}_2^2] \\ \!+\! \dfrac{4 \xi^2}{K \lambda} \!\log \!\left(\!\!1 \!+\! \dfrac{\lambda K}{4 \beta}\!\!\right) \!+\! 2 d\sigma^2,
    \end{multline}
where $\theta^*$ is the minimizer of $\L_{\D}(\theta)$ in $\C$.
%\note{FB: may be use two lines for the equation above.}
\end{lemmaE}
\begin{proofE}
We define $\eta_k$ as such 
\[
    \eta_k = \dfrac{1}{2 \beta + \frac{\lambda k}{2}}.
\]
Note that in particular that for $k \ge 0$, $\eta_k \le \frac{1}{2 \beta}$ and hence $ 2 \eta_k (\eta_k \beta - 1) \le - \eta_k$.

As detailed in the proof of Lemma \ref{lemma:emp_risk} we have the following:
$$\begin{array}{l}
    \E_\D[\norm{\theta_{k+1} \!-\! \theta^*}_2^2] \\
    \; \le \norm{\theta_k - \theta^*}_2^2 + 2 \eta_{k+1} (\eta_{k+1} \beta - 1) \langle \theta_k - \theta^*, \nabla \L_{\D}(\theta_k) \rangle \\
    \quad + 2 \eta_{k+1}^2 \E_\D[\norm{\nabla \L_{\B}(\theta^*) }_2^2] + 2 \eta_{k+1} \sigma^ 2 \E_\D[\norm{\gauss{}}_2^2]\\
    \quad +2 \langle \theta_k - \theta^*, \sqrt{2 \eta_{k+1} \sigma^2} \gauss{} \rangle \\
    \quad - 2 \eta_{k+1} \langle \nabla \L_{\D}(\theta_k), \sqrt{2 \eta_{k+1} \sigma^2} \gauss{} \rangle \\
    \quad \mathrm{\:using\:(\ref{eq:cocoercitivity})\:and\:}(\ref{eq:expected_coco})
\end{array}$$

By taking expectation again we derive
$$\begin{array}{ll}
    \E[\norm{\theta_{k+1} \!\!-\! \theta^*}_2^2]\!\!\!\!\! & \le \E[\norm{\theta_k - \theta^*}_2^2] \\
    & \!\!\!\! + 2 \eta_{k+1} (\eta_{k+1} \beta \!-\! 1) \E[\!\langle \nabla \L_{\D}(\theta_k\!), \theta_k \!-\! \theta^* \rangle] \\
    & \!\!\!\! + 2 \eta_{k+1}^2 \E[\norm{\nabla \L_{\B}(\theta^*) }_2^2] + 2 \eta_{k+1} \sigma^ 2 d \\
    & \!\!\!\! \le \E[\norm{\theta_k - \theta^*}_2^2] \\
    & \!\!\!\! - \eta_{k+1} \E[\langle \nabla \L_{\D}(\theta_k), \theta_k - \theta^* \rangle] \\
    & \!\!\!\! + 2 \eta_{k+1} (\eta_{k+1} \xi^2 + d \sigma^2)
\end{array}$$    

By optimality of $\theta^*$ and strong convexity arguments on $\L_\D$, we have
$$\begin{array}{ll}
\!\!- \langle \theta_k \!-\! \theta^*\!, \nabla \L_{\D}(\theta_k) \rangle \!\!\!\!\!
& \le \L_\D(\theta_k) \!-\! \L_\D(\theta^*\!) \!+\! \dfrac{\lambda}{2}\! \norm{\theta_k \!-\! \theta^*}_2^2 
\end{array}$$

Hence we have
$$\begin{array}{l}
    \E[\norm{\theta_{k+1} \!-\! \theta^*}_2^2] \\
    \; \le (1 \!-\! \dfrac{\eta_{k+1} \lambda}{2}) \E[\norm{\theta_k \!-\! \theta^*}_2^2] \!-\! \eta_{k+1} \E[\L_\D(\theta_k) \!-\! \L_\D(\theta^*\!)]\\
    \quad + 2 \eta_{k+1} (\eta_{k+1} \xi^2 + d \sigma^2)
\end{array}$$

Equivalently
$$\begin{array}{l}
    \E[\L_\D(\theta_k) \!-\! \L_\D(\theta^*)] \\
    \; \le - \dfrac{1}{\eta_{k+1}} \E[\norm{\theta_{k+1} \!-\! \theta^*}_2^2] + (\dfrac{1}{\eta_{k+1}} - \dfrac{\lambda}{2}) \E[\norm{\theta_{k} \!-\! \theta^*}_2^2] \\
    \quad + 2 (\eta_{k+1} \xi^2 + d \sigma^2)
\end{array}$$

Plugging in the definition of $\eta_k$ we have
$$\begin{array}{l}
    \E[\L_\D(\theta_k) \!-\! \L_\D(\theta^*)] \\
    \; \le - (2 \beta + \dfrac{\lambda}{2} (k+1)) \E[\norm{\theta_{k+1} \!-\! \theta^*}_2^2] \\
    \quad\:\: + (2 \beta + \dfrac{\lambda}{2} k) \E[\norm{\theta_{k} \!-\! \theta^*}_2^2] \qquad\qquad \\
    \quad\:\: + 2 (\eta_{k+1} \xi^2 + d \sigma^2)
\end{array}$$

We derive the average empirical risk
$$\begin{array}{l}
    \E[\frac{1}{K}\sum_{k=0}^{K-1} \L_{\D}(\theta_k\!) \!-\! \L_\D(\theta^*\!)] \\
    \; = \dfrac{1}{K}\sum_{k=0}^{K-1} \E[\L_{\D}(\theta_{k}\!) \!-\! \L_\D(\theta^*\!)] \\
    \; = - \dfrac{1}{K} (2 \beta + \dfrac{\lambda}{2}K) \E[\norm{\theta_{K} \!-\! \theta^*}_2^2] \\
    \quad + \dfrac{2 \beta}{K} \E[\norm{\theta_{0} \!-\! \theta^*}_2^2] + 2 d \sigma^2 + \dfrac{2 \xi^2}{K} \displaystyle{\sum_{k=1}^{K}} \eta_k \\
    \; \le \dfrac{2 \beta}{K} \E[\norm{\theta_{0} \!-\! \theta^*}_2^2] + 2 d \sigma^2  + \dfrac{2 \xi^2}{K} \displaystyle{\sum_{k=1}^{K}} \dfrac{1}{2 \beta + \frac{\lambda k}{2}} \\
    \; \le \dfrac{2 \beta}{K} \E[\norm{\theta_{0} \!-\! \theta^*}_2^2] + 2 d \sigma^2  + \dfrac{2 \xi^2}{K} \displaystyle{\int_{0}^{K}} \dfrac{1}{2 \beta + \frac{\lambda t}{2}} dt \\
     \; \le \dfrac{2 \beta}{K} \E[\norm{\theta_{0} \!-\! \theta^*}_2^2] + 2 d \sigma^2 + \dfrac{4 \xi^2}{K \lambda} \log \left(1 + \dfrac{\lambda K}{4 \beta} \right)
\end{array}$$

\end{proofE}

%The proof of the Lemma is given in Appendix \ref{app:utility_SGLD:decreasing_step}.

The term $\smash{\frac{\xi^2}{\lambda}}$ decreases roughly in $\smash{\frac{1}{K} \log(K)}$. The term $d \sigma^2$ still appears as in Lemma \ref{lemma:emp_risk} but without the $\smash{\frac{\beta}{\lambda}}$ factor, that we observe to be quite significant in Section \ref{sec:experiments}.

We then use this lemma to derive the following utility bound under $(\alpha, \varepsilon)$-Rényi differential privacy:

\begin{theoremE}[Utility bound for $(\alpha, \varepsilon)$-Rényi differential privacy with decreasing learning rate][category=th:utility_renyi_decreasing_eta]\label{th:utility_renyi_decreasing_eta}
    Let $\ell(\theta, \vect x)$ be an $L$-Lipschitz, $\lambda$-strongly convex and $\beta$-smooth loss function on closed convex set $\C$, then $\AlgoNoisySGD$ with start parameter $\smash{\theta_0 \sim \proj(\gauss{\frac{2\sigma^2}{\lambda}})}$ and \emph{decreasing} step-size $\smash{\eta_k = \frac{1}{2 \beta + \lambda k / 2}}$, satisfies $(\alpha, \varepsilon)$ Rényi differential privacy and
    \begin{equation*}
        \E[\frac{1}{K}\!\!\sum_{k=1}^K \!\L_{\D}(\theta_k\!) - \L_\D(\theta^*\!)] = O\left(\dfrac{ \alpha d L^2}{\varepsilon \lambda n^2} \right),
    \end{equation*}
    where $\sigma^2$ and $K$ are set as such:
    \begin{equation*}
        \sigma^2 = \dfrac{4 \alpha L^2}{\varepsilon \lambda  n^2}, \quad
        K = \max\! \left(\dfrac{\beta}{\lambda} \dfrac{\varepsilon n^2}{\alpha d}, \dfrac{\lambda}{\beta} \left(\dfrac{\varepsilon n^2}{\alpha d} \right)^2 \right).
    \end{equation*}
\end{theoremE}
\begin{proofE}
Since $\theta_0, \theta^* \in \C$ and $\norm{C}_2 \le \frac{2L}{\lambda}$, we have $\E[\norm{\theta_0 - \theta^*}_2^2] \le \frac{4L^2}{\lambda^2}$. Additionally, by Lipschitzness of the loss we have $\xi^2 \le L^2$.

We can use this and plug the value of $\sigma^2$ to rewrite Lemma \ref{lemma:emp_risk_decreasing_eta}:

\[
\begin{array}{ll}
    \E[\frac{1}{K}\!\!\sum_{k=1}^K \!\L_{\D}\!(\theta_k\!) \!-\! \L_\D\!(\theta^*\!)]\!\!\!\!
    & \!\le\! \dfrac{8 \alpha d L^2}{\varepsilon \lambda  n^2} + \dfrac{8 \beta L^2}{K \lambda^2} \\
    & \quad + \dfrac{L^2 }{\beta} \dfrac{4 \beta}{ K \lambda} \log(1 + \dfrac{K \lambda}{4 \beta}) \\ 
\end{array}
\]

Then, we use the following bound on the logarithmic function from \cite{mitrinovic1970analytic}:
\begin{equation*}
    \forall x > -1, \dfrac{1}{x} \log(1 + x) \le \dfrac{1}{\sqrt{1 + x}}
\end{equation*}
to derive:
\begin{equation*}
    \dfrac{L^2 }{\beta} \dfrac{4 \beta}{ K \lambda} \log(1 + \dfrac{K \lambda}{4 \beta}) \le
    \dfrac{L^2 }{\beta} \dfrac{1}{\sqrt{1 + \dfrac{K \lambda}{4 \beta}}}
\end{equation*}

We can now plug the value of $K$.

\[
\begin{array}{ll}
\E[\frac{1}{K}\!\!\sum_{k=1}^K \!\L_{\D}\!(\theta_k\!) \!-\! \L_\D\!(\theta^*\!)]\!\!\!\!
    & \!\le\! \dfrac{8 \alpha d L^2}{\varepsilon \lambda  n^2} + \dfrac{8 \beta L^2}{K \lambda^2} \\
    & \quad + \dfrac{L^2 }{\beta} \sqrt{\dfrac{4 \beta}{\lambda} \dfrac{\beta}{\lambda} \left(\dfrac{\alpha d}{\varepsilon n^2} \right)^2 } \\ 
    & \!\le\! \dfrac{18 \alpha d L^2}{\varepsilon \lambda  n^2}
\end{array}
\]

\end{proofE}

Compared to previous Theorem \ref{th:utility_renyi}, we improve the utility bound by a factor $\frac{\beta}{\lambda}$ which is non negligible in practice. However, the number of iterations $K$ can now evolve either in $n^2$ or in $n^4$ in the regime where $\smash{\frac{\varepsilon n^2}{\alpha d} > \big( \frac{\beta}{\lambda} \big)^2}$.

\begin{theoremE}[Utility bound for $(\epsilon, \delta)$-differential privacy with decreasing learning rate][category=th:utility_dwork_decreasing_eta]
  With the same conditions as is Theorem \ref{th:utility_renyi_decreasing_eta}, for $\epsilon \le 2 \log(1/\delta)$ and $\delta > 0$, $\AlgoNoisySGD$ satisfies $(\epsilon, \delta)$ differential privacy and
    \begin{equation*}
        \E[ \L_{\D}(\theta_K) - \L_{\D}(\theta^*)] = O\left(\dfrac{ d L^2 \log(1 / \delta)}{\epsilon^2\lambda n^2} \right) ,
    \end{equation*}
    where $\sigma^2$ and $K$ are set as such:
    \begin{align*}
        \sigma^2 &= \dfrac{8 L^2 (\epsilon + 2 \log(1/\delta)}{\epsilon^2 \lambda n^2} \\
        K &= \max \Bigg( 
        \dfrac{\beta}{\lambda}
        \dfrac{ \epsilon^2 n^2}{ 4 \log(1 / \delta) d},
        \dfrac{\lambda}{\beta}
        \left( \dfrac{ \epsilon^2 n^2}{ 4 \log(1 / \delta) d} \right)^2
    \Bigg).
    \end{align*}
\end{theoremE}
\begin{proofE}
By setting $\varepsilon = \frac{\epsilon}{2}$, we derive from Proposition \ref{prop:renyi_to_dp} :
\begin{equation*}
    \alpha = 1 + \frac{2}{\epsilon} \log(1 / \delta)
\end{equation*}.

We use this to rewrite the results from Theorem \ref{th:utility_renyi_decreasing_eta}:
\[
\begin{array}{ll}
    \E[\frac{1}{K}\!\!\sum_{k=1}^K \!\L_{\D}(\theta_k\!) - \L_\D(\theta^*\!)]\!\!\! &\le \dfrac{ 18 \alpha d L^2}{\varepsilon \lambda n^2}  \\
    & = \dfrac{ 18 d L^2}{ \lambda n^2} \dfrac{1 + \frac{2}{\epsilon} \log(1 / \delta)}{\frac{\epsilon}{2}}  \\ 
    & = \dfrac{ 18 d L^2}{ \lambda n^2} \dfrac{2 \epsilon + 4 \log(1 / \delta)}{\epsilon^2} \\
    & \le \dfrac{ 18 d L^2}{ \lambda n^2} \dfrac{8 \log(1 / \delta)}{\epsilon^2} \\
    & \quad \mathrm{using} \: \epsilon \le 2 \log(1/\delta)
\end{array}
\]
Similarly,
\[
\begin{array}{ll}
    \sigma^2\! &= \dfrac{4 L^2 \alpha}{\lambda n^2 \varepsilon} \\
    &= \dfrac{4 L^2}{\lambda n^2}  \dfrac{1 + \frac{2}{\epsilon} \log(1 / \delta)}{\frac{\epsilon}{2}} \\
    &= \dfrac{8 L^2 (\epsilon + 2 \log(1/\delta)}{\epsilon^2 \lambda n^2}
\end{array}
\]
\[
\begin{array}{ll}
    \dfrac{\epsilon n^2}{\alpha d}\! &= \dfrac{n^2}{ d} \dfrac{\frac{\epsilon}{2}}{1 + \frac{2}{\epsilon} \log(1 / \delta) } \\
    &= \dfrac{n^2}{ d} \dfrac{\epsilon^2}{ 2 \epsilon + 4 \log(1 / \delta) } \\
    &\le \dfrac{ \epsilon^2 n^2}{ 4 \log(1 / \delta) d} \\
\end{array}
\]
\[
\begin{array}{ll}
    K\! &= \max \left(\dfrac{\beta}{\lambda} \dfrac{\epsilon n^2}{\alpha d}, \dfrac{\lambda}{\beta} \left(\dfrac{\epsilon n^2}{\alpha d} \right)^2 \right) \\
    &\le \max \left( 
        \dfrac{\beta}{\lambda}
        \dfrac{ \epsilon^2 n^2}{ 4 \log(1 / \delta) d},
        \dfrac{\lambda}{\beta}
        \left( \dfrac{ \epsilon^2 n^2}{ 4 \log(1 / \delta) d} \right)^2
    \right) \\
\end{array}
\]
\end{proofE}

\section{Experiments: application to logistic regression}\label{sec:experiments}

We now propose an experimental evaluation of DP-SGLD and compare it to DP-SGD on a classification task using logistic regression on two vision datasets, CIFAR10 and Pneumonia, a dataset of chest X-ray images of pediatric pneumonia published by \citet{kermany2018identifying}. Details about the datasets and models are available in Appendix \ref{app:models_datasets}.

The task consists of pre-training a model (here AlexNet or ResNet18) on a dataset that will be considered \emph{public} (here we take CIFAR100 and Imagenet).
%The task is defined as such : a model (for example ResNet18) is pre-trained on a non-sensitive dataset (here CIFAR100) to be able to learn a model that extracts meaningful features.
Then, all layers of the model are freezed except for the last one which is retrained from scratch using a softmax loss function on a \emph{private} dataset (here CIFAR10 or Pneumonia). This corresponds to logistic regression and some regularization is added to guarantee strong convexity. Pre-training provides generic feature maps learned on a public dataset which improves the task accuracy without any compromise on the privacy.
%nd to leverage them using a smooth and strongly-convex optimization task that fits the hypothesis of DP-SGLD.

First, we formalize this setting and provide the smoothness and convexity constants for logistic regression. 
% We then derive from the results from Section \ref{sec:utility_SGLD} some utility guarantees for logistic regression using DP-SGLD. Finally, 
Second, we report accuracy achieved with DP-SGLD and compare it to DP-SGD for constant and decreasing step size $\eta$. % We use the implementation of the Opacus library for DP-SGD. 
%We also show that the possibility for standard DP-SGD to address non-smooth non-convex tasks, which materializes here with the ability to jointly train the classifier layer and the last layers of the feature extractor (i.e. perform fine-tuning), doesn't result in a noticeable utility increase.

\subsection{Smoothness and convexity of logistic regression}

For clarity, we replace the generic parameter $\theta$ with the single matrix $\mat W$ that it represents for logistic regression.

% The loss $\ell(\mat W, \vect x)$  with regularization writes:
% \begin{equation}\label{eq:log_reg}
%    \ell(\mat W, \vect x) = \log \big( \sigmoid(\mat W \vect x ) \big)_{y} + \lambda \norm{\mat W}_2^2
%\end{equation}
The loss with regularization writes $\ell(\mat W, \vect x) = \log \big( \sigmoid(\mat W \vect x ) \big)_{y} + \lambda \norm{\mat W}_2^2$
where $C$ is the number of classes, $y \in [1, C]$ is the label of sample $\vect x$, $\smash{\mat W \in \R^{C \times p}}$ and $\smash{\sigmoid : \R^C \mapsto \R^C}$ is the sigmoid function (not to be confused with the noise variance $\sigma$):
\[
    \big(\sigmoid(\vect z) \big)_i = \dfrac{e^{z_i}}{\sum_{j = 1}^C e^{z_j}}
    , \: \forall i = 1,..,C \:, \vect z \in \R^C
\]

%We derive the global empirical loss $\L_\D$:
%\begin{equation}
%    \L_\D(\mat W) = \dfrac{1}{n} \sum_{j = 1}^n \log \big( \sigmoid(\mat W \vect x_j ) \big)_{y_j} + \lambda %\norm{\mat W}_2^2
%\end{equation}

\begin{lemmaE}[Convexity and smoothness constants for regularized logistic regression][category=lemma:smoothness_cvx_log_reg]
    Let $\ell(\mat W, \vect x)$ be defined as above. Then $\ell$ is $\lambda$-strongly convex and $\beta$-smooth, with
    \begin{align*}
        \beta = \dfrac{1}{2} \eigenmax \Big( \dfrac{1}{n} \sum_{i=1}^{n} \vect x_i \vect x_i^\top \Big) + \lambda
    \end{align*}
    where $\eigenmax$ refers to the maximum eigenvalue. %\note{FB: don't you also need $L$?}
\end{lemmaE}

\begin{proofE}
By plugging in the definition of the sigmoid, the loss $\ell$ also writes:
\begin{equation*}
    \ell(\mat W, \vect x) = \log \dfrac{\exp \big((\mat W \vect x )_{y}\big)}{ \sum_{i=1}^C \exp \big((\mat W \vect x )_i \big)} + \lambda \norm{\mat W}_2^2,
\end{equation*}
where $\mat W \in \R^{C \times p}$ also writes $[\vect w_1^\top, \dots, \vect w_C^\top]^\top$.

Let's define for i = 1, \dots, C
\begin{align*}
        p_i = P(y_i = 1 | \vect x, \mat W) 
        & = \dfrac{\exp \big((\mat W \vect x )_{i}\big)}{ \sum_{c=1}^C \exp \big((\mat W \vect x )_c \big)} \\
        & = \dfrac{\exp (\vect w_i^\top \vect x )}{ \sum_{c=1}^C \exp (\vect w_c^\top \vect x )}
\end{align*}
We can rewrite the loss $\ell$ as such
\begin{equation*}
    \ell(\mat W, \vect x) = \log \prod_{i=1}^C p_i^{y_i} + \lambda \norm{\mat W}_2^2
\end{equation*}
where the label $y$ is now one-hot encoded: $y = (y_0, \dots, y_C)$. 

Following the work of \cite{bohning1992multinomial}, we have:
\begin{equation*}
    \nabla^2 \ell(\mat W, \vect x) = (\mat D(\vect p) - \vect p \vect p^\top) \otimes \vect x \vect x^\top + \lambda \mat I_C
\end{equation*}
where $\mat D(\vect p) = \mat I_C \vect p$.

We derive
\begin{equation*}
    \nabla^2 \L_\D(\mat W) = \dfrac{1}{n} \sum_{j=1}^{n} (\mat D(\vect p_j) - \vect p_j \vect p_j^\top) \otimes \vect x_j \vect x_j^\top  + \lambda \mat I_C
\end{equation*}
where $\vect p_j$ corresponds to $\vect p$ conditioned with $\vect x_j$.

As shown in \cite{krishnapuram2005sparse}, $\nabla^2 \L_\D(\mat W)$ satisfies the following
\begin{equation*}
    \lambda \mat I_C \preceq \nabla^2 \L_\D(\mat W) \preceq \dfrac{1}{2} \big(\mat I_C - \dfrac{1}{C} \mat 1_C \mat 1_C^\top \big) \otimes \dfrac{1}{n} \mat X \mat X^\top + \lambda \mat I_C
\end{equation*}
where $\mat X = [\vect x_1, \dots, \vect x_n]^\top \in \R^{n \times d}$.
In particular, we deduce:
\begin{align*}
    \beta &= \eigenmax(\nabla^2 \L_\D(\mat W)) \\
    & \le \eigenmax \left(\dfrac{1}{2} \big(\mat I_C - \dfrac{1}{C} \mat 1_C \mat 1_C^\top \big) \otimes \dfrac{1}{n} \mat X \mat X^\top + \right) + \lambda \\
    & = \max_{\lambda_{\mathsf{eig}}, \lambda_{\mathsf{eig}}'} \lambda_{\mathsf{eig}} \left( \dfrac{1}{2} \big(\mat I_C - \dfrac{1}{C} \mat 1_C \mat 1_C^\top \big) \right) \lambda'_{\mathsf{eig}} \left( \dfrac{1}{n} \mat X \mat X^\top \right) + \lambda \\
    & \le \dfrac{1}{2n} \eigenmax\left( \mat X \mat X^\top \!\right) + \lambda
\end{align*}

\end{proofE}

\vspace{-0.1cm}
\subsection{Experimental utility of logistic regression}

%We also show that the possibility for standard DP-SGD to address non-smooth non-convex tasks, which materializes here with the ability to jointly train the classifier layer and the last layers of the feature extractor (i.e. perform fine-tuning), doesn't result in a noticeable utility increase.

We compare our DP-SGLD algorithm with the standard DP-SGD from \citet{abadi2016deep} implemented in Opacus and with the baseline SGD without DP on several vision tasks.
%We have conducted several learning tasks involving logistic regression, where we compare our DP-SGLD algorithm with the standard DP-SGD approach from \cite{abadi2016deep} implemented in Opacus, and with the baseline SGD method without DP. 
In particular, we study the case where the step size is constant $\smash{\eta = \frac{1}{2 \beta}}$ and where it is decreasing as follows : $\smash{\eta_k = \frac{1}{2 \beta + \lambda k / 2}}$. To be able to provide somewhat comparable results, all methods (DP-SGLD, DP-SGD and No-DP) use the same step size, number of training epochs, privacy budget $(\epsilon, \delta) = (1.0, 10^{-5})$ when applicable and use no momentum. Other hyperparameters are tuned to provide optimal accuracy for each method and are provided in the source code included in the submission.
% which will be made open-source\footnote{\url{https://github.com/xxx}}. TODO ADD BACK LINK

Results are given in Table \ref{table:acc} for constant step size and in Table \ref{table:acc_decreasing_stepsize} for decreasing step size. The model indicated is the feature extraction model, which is pre-trained on CIFAR100 when the task is on CIFAR10 and on Imagenet when the task is on Pneumonia. Only its last layer is re-trained using logistic regression. As the tables show, DP-SGLD outperforms standard DP-SGD for the tasks considered and considerably reduces the gap in accuracy compared to SGD without differential privacy. However, such results need to be taken cautiously before drawing conclusions since this task is strongly convex and smooth while DP-SGD also applies to non-convex tasks.

To better understand the effect of clamping DP-SGD to smooth and strongly convex tasks, we repeat the first experiment of Table \ref{table:acc}, but instead of leveraging only the last layer for DP-SGD, we also unfreeze the last and fourth block of the ResNet18 architecture, composed notably of 5 convolutional layers.
%This means that we allow for more finetuning of the model to the target task.
Results for DP-SGD and SGD without DP provided in Table \ref{table:acc_finetuning} show that while SGD benefits from this fine-tuning and increases accuracy from 70.7\% to 77.0\%, DP-SGD does not improve and accuracy even decreases marginally from 68.0\% to 67.8\%. Such observation aligns with \citet{dp_alexnet_moment} in the sense that basic models like logistic regression currently are competitive compared to deeper models when trained with differential privacy, which underlines the importance of studying classical tasks like training smooth and strongly convex objectives.

Last, we provide in Table \ref{table:parameter_info} the experimental value of some parameters, in light with comments made after Lemma \ref{lemma:emp_risk_decreasing_eta} about the value of $\frac{\beta}{\lambda}$ and after Theorem \ref{th:utility_renyi_decreasing_eta} about the dependence in $n^2$ or $n^4$ of $K$, depending of the ratio $\smash{\frac{\varepsilon n^2}{\alpha d} / \big( \frac{\beta}{\lambda} \big)^{2}}$. As we show, this ratio is of magnitude $10^{-3}$ or less which shows that the evolution of $K$ is quadratic in $n$.

\begin{table} 
\caption{Accuracy (in \%) of logistic regression using SGD with a constant learning rate. \label{table:acc}}
% Results are given with --seed=1,2,3,4,5
\begin{center}
\begin{tabular}{l l l c c c}
\hline
 Method & Dataset & Model & \!\!\!\!Epochs\!\!\!\! & $\epsilon$  & Acc.  \\ 
\hline
% python main.py --model resnet18 --dataset cifar10 --lr 0.018 --epochs 30 --lambd 0.001 --max_grad_norm 15 --sigma 0.0022 --langevin
% 70.31 = 70.31, 70.32, 70.24, 70.29, 70.37
% α = 25
DP-SGLD\!\! & CIFAR10  & Resnet18 & 30  & 1.0 & 70.3 \\
% python main.py --model resnet18 --dataset cifar10 --lr 0.018 --epochs 30 --lambd 0.001 --max_grad_norm 5 --noise_multiplier 1.33 --renyi
% 68.03 = 68.3, 67.94, 67.79, 67.88, 68.24
% α = 17.0
DP-SGD  & CIFAR10   & Resnet18  & 30 & 1.0 & 68.0 \\ 
% python main.py --model resnet18 --dataset cifar10 --lr 0.018 --epochs 30 --lambd 0.001
% 70.66 = 70.55, 70.64, 70.71, 70.67, 70.73
No DP & CIFAR10   & Resnet18  & 30 & - & 70.7 \\
\hline
% python main.py --model alexnet --dataset cifar10 --lr 0.0038 --epochs 30 --lambd 0.005 --max_grad_norm 15 --sigma 0.001 --langevin
% 57.50 = 57.61, 57.43, 57.54, 57.48, 57.42
DP-SGLD\!\!  & CIFAR10 & Alexnet & 30 & 1.0 & 57.5 \\
% python main.py --model alexnet --dataset cifar10 --lr 0.0038 --epochs 30 --lambd 0.001 --max_grad_norm 8 --noise_multiplier 1.33 --renyi
% 56.40 = 56.36, 56.32, 56.45, 56.4, 56.47
DP-SGD & CIFAR10 & Alexnet & 30 & 1.0 & 56.4 \\
% python main.py --model alexnet --dataset cifar10 --lr 0.0038  --epochs 30 --lambd 0.001 
% 57.74 = 57.7, 57.79, 57.83, 57.65, 57.72
No DP & CIFAR10 & Alexnet & 30 & - & 57.7 \\
\hline
% python main.py --model resnet18 --dataset pneumonia --lr 0.0014 --epochs 50 --lambd 0.005 --max_grad_norm 18 --sigma 0.0082 --delta 0.0001 --langevin
% 58.75 = 58.97, 58.49, 59.62, 58.17, 58.49
DP-SGLD\!\! & Pneumonia\!\!\!\! & Resnet18 & 50 & 1.0 & 58.8 \\
% python main.py --model resnet18 --dataset pneumonia --lr 0.0014 --epochs 50 --lambd 0.005 --max_grad_norm 25 --noise_multiplier 4.03 --delta 0.0001 --renyi
% 58.75 = 58.97, 59.13, 58.81, 58.33, 58.49
DP-SGD & Pneumonia\!\!\!\! & Resnet18 & 50 & 1.0 & 58.8 \\
% python main.py --model resnet18 --dataset pneumonia --lr 0.0014 --epochs 50 --lambd 0.005
% 59.33 = 59.46, 59.78, 59.94, 59.13, 58.33
No DP & Pneumonia\!\!\!\! & Resnet18 & 50 & - & 59.3 \\
\hline
\end{tabular}
\end{center}
\end{table}

\begin{table} 
\caption{Accuracy (in \%) of logistic regression using SGD with a decreasing learning rate. \label{table:acc_decreasing_stepsize}}
\begin{center}
\begin{tabular}{l l l c c c}
\hline
 Method & Dataset & Model & \!\!\!\!Epochs\!\!\!\! & $\epsilon$  & Acc.  \\ 
 \hline
 % python main.py --model resnet18 --dataset cifar10 --epochs 30 --lambd 0.001 --max_grad_norm 20 --sigma 0.0021 --langevin --decreasing
 % 70.10 = 70.04, 70.16, 70.05, 70.16, 70.07
 DP-SGLD\!\! & CIFAR10  & Resnet18 & 30  & 1.0 & 70.1 \\
 % python main.py --model resnet18 --dataset cifar10 --epochs 30 --lambd 0.001 --max_grad_norm 7 --noise_multiplier 1.33 --renyi --decreasing
 % 68.09 = 68.3, 68.05, 67.82, 67.97, 68.29
 DP-SGD & CIFAR10 & Resnet18 & 30 & 1.0 & 68.1 \\
 % python main.py --model resnet18 --dataset cifar10 --epochs 30 --lambd 0.001 --decreasing 
 % 70.24 = 70.19, 70.18, 70.07, 70.29, 70.45
 No DP & CIFAR10 & Resnet18 & 30 & - & 70.2 \\
\hline
 % python main.py --model alexnet --dataset cifar10 --epochs 30 --lambd 0.001 --max_grad_norm 20 --sigma 0.00095 --langevin --decreasing 
 % 57.33 = 57.37, 57.23, 57.37, 57.46, 57.22
 DP-SGLD\!\! & CIFAR10 & Alexnet & 30 & 1.0 & 57.3 \\
 % python main.py --model alexnet --dataset cifar10 --epochs 30 --lambd 0.001 --max_grad_norm 15 --noise_multiplier 1.33 --renyi --decreasing
 % 56.35 = 56.38, 56.17, 56.44, 56.39, 56.36
 DP-SGD & CIFAR10 & Alexnet & 30 & 1.0 & 56.4 \\
 % python main.py --model alexnet --dataset cifar10 --epochs 30 --lambd 0.0005 --decreasing
 % 57.61 = 57.68, 57.58, 57.65, 57.66, 57.48
 No DP & CIFAR10 & Alexnet & 30 & - & 57.6 \\
\hline
 % python main.py --model resnet18 --dataset pneumonia --lr 0.0014 --epochs 50 --lambd 0.005 --max_grad_norm 18 --sigma 0.0082 --delta 0.0001 --langevin --decreasing
 % 58.75 = 58.97, 58.49, 59.62, 58.17, 58.49
DP-SGLD\!\! & Pneumonia\!\!\!\! & Resnet18 & 50 & 1.0 & 58.8 \\
 % python main.py --model resnet18 --dataset pneumonia --lr 0.0014 --epochs 50 --lambd 0.005 --max_grad_norm 25 --noise_multiplier 4.03 --delta 0.0001 --renyi --decreasing
 % 58.75 = 58.97, 59.13, 58.81, 58.33, 58.49
DP-SGD & Pneumonia\!\!\!\! & Resnet18 & 50 & 1.0 & 58.8 \\
 % python main.py --model resnet18 --dataset pneumonia --lr 0.0014 --epochs 50 --lambd 0.005 --decreasing
 % 59.33 = 59.46, 59.78, 59.94, 59.13, 58.33
No DP & Pneumonia\!\!\!\! & Resnet18 & 50 & - & 59.3 \\
\hline
\end{tabular}
\end{center}
\end{table}

\begin{table} 
\caption{Accuracy (in \%) when fine-tuning ResNet18. \label{table:acc_finetuning}}
\begin{center}
\begin{tabular}{l l l c c c}
\hline
 Method & Dataset & Model & \!\!\!\!Epochs\!\!\!\! & $\epsilon$  & Acc.  \\ 
\hline
% python main.py --model resnet18-finetuning --dataset cifar10 --lr 0.018 --epochs 30 --lambd 0.001 --max_grad_norm 5 --noise_multiplier 1.33 --renyi
% 67.81 = 67.56, 67.85, 68.12, 67.83, 67.67
DP-SGD  & CIFAR10   & Resnet18  & 30 & 1.00 & 67.8 \\ 
% python main.py --model resnet18-finetuning --dataset cifar10 --lr 0.018 --epochs 30 --lambd 0.001
% 77.03 = 76.89, 77.18, 76.8, 77.3, 76.98
No DP & CIFAR10   & Resnet18  & 30 & - & 77.0 \\
\hline
\end{tabular}
\end{center}
\end{table}

\begin{table} 
\caption{Value of some parameters used for DP-SGLD.}\label{table:parameter_info}
\begin{center}
\begin{tabular}{l l c c c}
\hline
  Dataset & Model & $\beta$ & $\frac{\beta}{\lambda}$ & $\textcolor{white}{\Big|} \frac{\varepsilon n^2}{\alpha d} \big / \! \big( \frac{\beta}{\lambda} \big)^2$  \\ 
 \hline
 CIFAR10  & Resnet18 & 55 & 5.5 $\times 10^{4}$ & 3.2 $\times  10^{-3}$ \\
 CIFAR10 & Alexnet & 259 & 2.6 $\times 10^{5}$ & 1.5 $\times  10^{-4}$ \\
 Pneumonia & Resnet18 & 354 & 7.1 $\times 10^{4}$ & 6.8 $\times 10^{-5}$ \\
\hline
\end{tabular}
\end{center}
\end{table}

\vspace{-0.1cm}
\section{Conclusion}

We have extended the theoretical framework of \citet{chourasia2021differential} to provide a differential privacy analysis of noisy stochastic gradient descent based on Langevin diffusion (DP-SGLD) with arbitrary step size. 
%Under similar hypothesis, %namely hiding the model iterates until the last step and using a smooth and strongly convex objective, 
%we derive the same exponentially fast convergence of the privacy leakage and tight utility bounds.
Although our experiments already show the practical utility of our results, relaxing the smoothness and strong convexity hypothesis remains an open challenge and would pave the way for wide adoption by data scientists.

\section*{Acknowledgments}

We would like to thank Pierre Tholoniat for the helpful discussions throughout this project. We are also grateful for the long-standing support of the OpenMined community and in particular its dedicated cryptography team.

This work was supported in part by the French-German Project CRYPTO4GRAPH-AI and by PRAIRIE, the PaRis Artificial Intelligence Research InstitutE.

\newpage
% In the unusual situation where you want a paper to appear in the
% references without citing it in the main text, use \nocite

\bibliography{references}

\begin{thebibliography}{21}
\providecommand{\natexlab}[1]{#1}
\providecommand{\url}[1]{\texttt{#1}}
\expandafter\ifx\csname urlstyle\endcsname\relax
  \providecommand{\doi}[1]{doi: #1}\else
  \providecommand{\doi}{doi: \begingroup \urlstyle{rm}\Url}\fi

\bibitem[Abadi et~al.(2016)Abadi, Chu, Goodfellow, McMahan, Mironov, Talwar,
  and Zhang]{abadi2016deep}
Abadi, M., Chu, A., Goodfellow, I., McMahan, H.~B., Mironov, I., Talwar, K.,
  and Zhang, L.
\newblock Deep learning with differential privacy.
\newblock In \emph{Proceedings of the SIGSAC conference on computer and
  communications security}, pp.\  308--318, 2016.

\bibitem[B{\"o}hning(1992)]{bohning1992multinomial}
B{\"o}hning, D.
\newblock Multinomial logistic regression algorithm.
\newblock \emph{Annals of the institute of Statistical Mathematics},
  44\penalty0 (1):\penalty0 197--200, 1992.

\bibitem[Chourasia et~al.(2021)Chourasia, Ye, and
  Shokri]{chourasia2021differential}
Chourasia, R., Ye, J., and Shokri, R.
\newblock Differential privacy dynamics of langevin diffusion and noisy
  gradient descent.
\newblock \emph{Advances in Neural Information Processing Systems}, 2021.

\bibitem[Dwork et~al.(2010)Dwork, Rothblum, and Vadhan]{dwork2010boosting}
Dwork, C., Rothblum, G.~N., and Vadhan, S.
\newblock Boosting and differential privacy.
\newblock In \emph{Annual Symposium on Foundations of Computer Science}, pp.\
  51--60. IEEE, 2010.

\bibitem[Dwork et~al.(2014)Dwork, Roth, et~al.]{dwork2014algorithmic}
Dwork, C., Roth, A., et~al.
\newblock The algorithmic foundations of differential privacy.
\newblock \emph{Foundations and Trends in Theoretical Computer Science},
  9\penalty0 (3-4):\penalty0 211--407, 2014.

\bibitem[Gross(1975)]{gross1975logarithmic}
Gross, L.
\newblock Logarithmic {S}obolev inequalities.
\newblock \emph{American Journal of Mathematics}, 97\penalty0 (4):\penalty0
  1061--1083, 1975.

\bibitem[He et~al.(2016)He, Zhang, Ren, and Sun]{he2016resnet}
He, K., Zhang, X., Ren, S., and Sun, J.
\newblock Deep residual learning for image recognition.
\newblock In \emph{Proceedings of the Conference on Computer Vision and Pattern
  Recognition}, pp.\  770--778, 2016.

\bibitem[Kermany et~al.(2018)Kermany, Goldbaum, Cai, Valentim, Liang, Baxter,
  McKeown, Yang, Wu, Yan, et~al.]{kermany2018identifying}
Kermany, D.~S., Goldbaum, M., Cai, W., Valentim, C.~C., Liang, H., Baxter,
  S.~L., McKeown, A., Yang, G., Wu, X., Yan, F., et~al.
\newblock Identifying medical diagnoses and treatable diseases by image-based
  deep learning.
\newblock \emph{Cell}, 172\penalty0 (5):\penalty0 1122--1131, 2018.

\bibitem[Knott et~al.(2021)Knott, Venkataraman, Hannun, Sengupta, Ibrahim, and
  van~der Maaten]{knott2021crypten}
Knott, B., Venkataraman, S., Hannun, A., Sengupta, S., Ibrahim, M., and van~der
  Maaten, L.
\newblock Crypten: Secure multi-party computation meets machine learning.
\newblock \emph{Advances in Neural Information Processing Systems}, 2021.

\bibitem[Krishnapuram et~al.(2005)Krishnapuram, Carin, Figueiredo, and
  Hartemink]{krishnapuram2005sparse}
Krishnapuram, B., Carin, L., Figueiredo, M.~A., and Hartemink, A.~J.
\newblock Sparse multinomial logistic regression: Fast algorithms and
  generalization bounds.
\newblock \emph{IEEE Transactions on Pattern Analysis and Machine
  Intelligence}, 27\penalty0 (6):\penalty0 957--968, 2005.

\bibitem[Krizhevsky et~al.(2012)Krizhevsky, Sutskever, and
  Hinton]{krizhevsky2012alexnet}
Krizhevsky, A., Sutskever, I., and Hinton, G.~E.
\newblock Imagenet classification with deep convolutional neural networks.
\newblock In \emph{Advances in Neural Information Processing Systems}, pp.\
  1097--1105, 2012.

\bibitem[Krizhevsky et~al.(2014)Krizhevsky, Nair, and
  Hinton]{krizhevsky2014cifar}
Krizhevsky, A., Nair, V., and Hinton, G.
\newblock The {CIFAR}-10 dataset.
\newblock \emph{online: \url{cs.toronto.edu/~kriz/cifar.html}}, 55, 2014.

\bibitem[Mironov(2017)]{mironov2017renyi}
Mironov, I.
\newblock R{\'e}nyi differential privacy.
\newblock In \emph{Computer Security Foundations Symposium (CSF)}, pp.\
  263--275. IEEE, 2017.

\bibitem[Mitrinovic \& Vasic(1970)Mitrinovic and Vasic]{mitrinovic1970analytic}
Mitrinovic, D.~S. and Vasic, P.~M.
\newblock \emph{Analytic Inequalities}, volume~1.
\newblock Springer, 1970.

\bibitem[Roux et~al.(2012)Roux, Schmidt, and Bach]{NIPS2012_905056c1}
Roux, N., Schmidt, M., and Bach, F.
\newblock A stochastic gradient method with an exponential convergence rate for
  finite training sets.
\newblock In \emph{Advances in Neural Information Processing Systems}, pp.\
  3167--3175, 2012.

\bibitem[Ryffel et~al.(2022)Ryffel, Tholoniat, Pointcheval, and
  Bach]{ryffel2022ariann}
Ryffel, T., Tholoniat, P., Pointcheval, D., and Bach, F.
\newblock Ariann: Low-interaction privacy-preserving deep learning via function
  secret sharing.
\newblock \emph{Proceedings on Privacy Enhancing Technologies}, 2022.

\bibitem[Shokri et~al.(2017)Shokri, Stronati, Song, and
  Shmatikov]{shokri2017membership}
Shokri, R., Stronati, M., Song, C., and Shmatikov, V.
\newblock Membership inference attacks against machine learning models.
\newblock In \emph{Symposium on Security and Privacy (SP)}, pp.\  3--18, 2017.

\bibitem[Tram{\`{e}}r \& Boneh(2021)Tram{\`{e}}r and Boneh]{dp_alexnet_moment}
Tram{\`{e}}r, F. and Boneh, D.
\newblock Differentially private learning needs better features (or much more
  data).
\newblock In \emph{9th International Conference on Learning Representations,
  {ICLR}}, 2021.

\bibitem[Wagh et~al.(2021)Wagh, Tople, Benhamouda, Kushilevitz, Mittal, and
  Rabin]{wagh2021falcon}
Wagh, S., Tople, S., Benhamouda, F., Kushilevitz, E., Mittal, P., and Rabin, T.
\newblock {FALCON: Honest-Majority Maliciously Secure Framework for Private
  Deep Learning}.
\newblock \emph{Proceedings on Privacy Enhancing Technologies}, 2021.

\bibitem[Welling \& Teh(2011)Welling and Teh]{welling2011bayesian}
Welling, M. and Teh, Y.~W.
\newblock Bayesian learning via stochastic gradient {L}angevin dynamics.
\newblock In \emph{Proceedings of the International conference on machine
  learning}, pp.\  681--688. Citeseer, 2011.

\bibitem[Yousefpour et~al.(2021)Yousefpour, Shilov, Sablayrolles, Testuggine,
  Prasad, Malek, Nguyen, Ghosh, Bharadwaj, Zhao, Cormode, and Mironov]{opacus}
Yousefpour, A., Shilov, I., Sablayrolles, A., Testuggine, D., Prasad, K.,
  Malek, M., Nguyen, J., Ghosh, S., Bharadwaj, A., Zhao, J., Cormode, G., and
  Mironov, I.
\newblock Opacus: User-friendly differential privacy library in pytorch.
\newblock \emph{arXiv preprint arXiv:2109.12298}, 2021.

\end{thebibliography}
\bibliographystyle{lib/icml2021}

\newpage

\appendix

\section{(Proofs) Privacy analysis of noisy stochastic gradient descent}\label{app:privacy_SGLD}

\printProofs[lemma:sensitivity]

\printProofs[th:rdp_noisysgd_clsi]

\section{(Proofs) Utility analysis for noisy stochastic gradient descent}\label{app:utility_SGLD}

\subsection{(Proofs) Fixed learning rate $\eta$}\label{app:utility_SGLD_fixed_step}

\printProofs[lemma:emp_risk]

\printProofs[th:utility_renyi]

\printProofs[th:utility_dwork]

\subsection{(Proofs) Decreasing step size $\eta_k$}\label{app:utility_SGLD_decreasing_step}

\printProofs[lemma:emp_risk_decreasing_eta]

\printProofs[th:utility_renyi_decreasing_eta]

\printProofs[th:utility_dwork_decreasing_eta]

\section{(Proofs) Experiments: application to logistic regression}\label{app:experiments}

\printProofs[lemma:smoothness_cvx_log_reg]

%\printProofs[lemma:emp_risk_log_reg]

\section{Datasets and models}\label{app:models_datasets}

\subsection*{Datasets}\label{app:datasets}

We have selected for our experiments two datasets commonly used for training image classification models: CIFAR-10 and CIFAR-100 \cite{krizhevsky2014cifar}, and also a dataset with healthcare data which can more closely mimic a scenario where we care about training a model on sensitive data: Pneumonia \cite{kermany2018identifying}. 

\textbf{CIFAR} 
CIFAR-10 and CIFAR-100 \cite{krizhevsky2014cifar} both consist of 50,000 images in the training set and 10,000 in the test set. They are respectively composed of 10 and 100
different balanced classes (such as airplanes, dogs, horses, etc.) and each image
consists of a colored 32$\times$32 image. The datasets are disjoints, which allows us to pretrain our models AlexNet and Resnet18 on CIFAR-100 and consider it \emph{public} pre-training before performing logistic regression on CIFAR10.

\textbf{Pneumonia}
Pneumonia is a dataset of chest X-ray images of pediatric pneumonia that was published by \cite{kermany2018identifying}. It is composed of 5163 training and 624 test non-colored images of varying sizes. Images are divided in 3 classes: bacterial (26\%), normal (48\%) and viral (26\%). It provides an interesting use case as it is a relatively small dataset and is composed of healthcare data.

\subsection*{Models}\label{appendix:NNarchi}

We have selected 2 models for our experimentations. 
    
\textbf{AlexNet}
AlexNet is the famous winner of the 2012 ImageNet ILSVRC-2012 competition \cite{krizhevsky2012alexnet}. It has 5 convolutional layers and 3 fully connected layers and it can use batch normalization layers for stability and efficient training.
    
\textbf{ResNet18}
ResNet18 \cite{he2016resnet} is the runner-up of the ILSVRC-2015 competition. It is a convolutional neural network that is 18 layers deep, and has 11.7M parameters. It uses batch normalisation layers, but as only the last layer is retrained with differential privacy, we need not replace those layers with group normalisation.

\end{document}